\documentclass{article}
\usepackage{iclr2027_conference,times}
\iclrfinalcopy
\usepackage[T1]{fontenc}
\usepackage{amsmath,amssymb,amsthm,mathtools}
\usepackage{graphicx,booktabs,microtype,enumitem}
\usepackage{xcolor}
\usepackage{hyperref,url}
\hypersetup{hidelinks,
 pdftitle={Gaussian Limits for SGD without Stationary Moments},pdfauthor={Xiaoli Li and Wei Biao Wu}}
\newtheorem{theorem}{Theorem}[section]
\newtheorem{lemma}[theorem]{Lemma}
\newtheorem{proposition}[theorem]{Proposition}
\newtheorem{corollary}[theorem]{Corollary}
\theoremstyle{definition}
\newtheorem{assumption}{Assumption}
\newtheorem{remark}[theorem]{Remark}
\newcommand{\E}{\mathbb E}
\newcommand{\Prob}{\mathbb P}
\newcommand{\R}{\mathbb R}
\newcommand{\F}{\mathcal F}
\newcommand{\G}{\mathcal G}
\newcommand{\Law}{\mathcal L}
\newcommand{\Proj}{\mathcal P}
\newcommand{\norm}[1]{\left\|#1\right\|}

\newcommand{\trans}{^{\mathsf T}}
\newcommand{\as}{\text{a.s.}}

\newcommand{\TV}{\mathrm{TV}}
\DeclareMathOperator{\Cov}{Cov}
\DeclareMathOperator{\tr}{tr}

\title{Gaussian Limits for SGD\\without Stationary Moments}
\author{%
Xiaoli Li\\
University of Chicago\\
\texttt{xiaolil@uchicago.edu}
\And
Wei Biao Wu\\
University of Chicago\\
\texttt{wbwu@uchicago.edu}
}

\newcommand{\ind}{\mathbf 1}

\begin{document}
\maketitle
\lhead{Preprint}

\begin{abstract}
Temporal dependence can separate the Gaussian approximation of stochastic gradient descent from its stationary moments. For unmodified least-squares SGD, we construct a design with standard Gaussian marginals whose stationary error has every positive moment infinite. Independent observations with the same marginals instead give finite stationary variance. Both regimes retain a Gaussian small-step limit. Our general theory establishes pathwise contraction from a finite second design moment, then uses score cancellation and localization to obtain stationary Gaussian and Ornstein--Uhlenbeck limits. Independent Gaussian regression errors yield an exact conditional Gaussian law and total-variation convergence under the same design integrability. Stronger design conditions identify a positive first-order total-variation constant and a deterministic covariance correction with $o(a)$ error. A scalar coverage expansion translates this correction into its inference consequence. Experiments examine distributional error, coverage, and calibration with dependent scores. Together, these results establish precise probability-law approximation beyond moment-based stationary analysis.
\end{abstract}

\section{Introduction}
Constant-step stochastic gradient descent (SGD) keeps adapting to a stream of observations. Its stationary last iterate describes the fluctuations that remain after initialization is forgotten. Gaussian and diffusion approximations connect these fluctuations to the learning rate and support uncertainty quantification \citep{kushner1981,mandt2017,li2019,dieuleveut2020,chen2021}. We ask how far this description extends when the stationary iterate has no positive moments.

We study unmodified least-squares SGD,
\begin{equation}\label{eq:model}
Y_i=X_i\trans\beta^*+Z_i,\qquad
\beta_i=\beta_{i-1}+aX_i(Y_i-X_i\trans\beta_{i-1}),
\end{equation}
with a fixed step $a>0$ during each run and stationary dependent observations. An extreme covariate can persist across updates and repeatedly amplify an earlier regression error.

Our clearest example holds the design marginal fixed at $N(0,1)$ and refreshes its value with probability $1-\rho$ at each step. Take iid centered errors of variance $\sigma^2>0$, independent of the design. Writing $e_0^{(a)}$ for the stationary error, for $0<a<2/3$ we obtain
\[
\rho=0:\quad \E(e_0^{(a)})^2=\frac{a\sigma^2}{2-3a}<\infty,
\qquad
0<\rho<1:\quad \E|e_0^{(a)}|^r=\infty\quad(r>0).
\]
With independent Gaussian regression errors, both regimes share a Gaussian small-step limit and the same first-order correction. Temporal persistence changes moment stability while preserving a precise probability-law approximation.

Total variation (TV) controls absolute errors in event probabilities, whereas moments integrate unbounded functions. Rare design environments can dominate every stationary moment while contributing little to a bounded probability distance. With independent Gaussian errors, conditioning on the entire design gives an exact Gaussian law with a vanishing random covariance perturbation. This identity explains the coexistence formalized in Corollary~\ref{cor:quantitative-coexistence}.

\paragraph{Gaussian limits and quantitative approximation.}
Constant-gain limits \citep{kushner1981} and diffusion approximations \citep{li2019} establish the Gaussian/Lyapunov framework. Stationary moment expansions \citep{dieuleveut2020} and distributional characterizations \citep{chen2021} connect it to constant-step SGD. For the finite-history last iterate, \citet{khusainov2025} obtain convex-set Gaussian bounds under iid sampling, bounded design and a finite third score moment. Our Gaussian-error branch permits dependent unbounded designs and gives a sharp TV expansion beyond every positive-order stationary moment class. The error-law assumptions and approximation metrics differ.

Convergence to an invariant law at a fixed step \citep{merad2023} and Gaussian approximation as $a\downarrow0$ are distinct limits; our finite-history bound connects them. Wasserstein and tail bounds \citep{wang2026} address moment-controlled settings. Averaged-iterate CLTs \citep{yu2021,mou2020} concern the covariance of accumulated iterates, while our correction concerns the stationary last iterate.

\paragraph{Product stability under temporal dependence.}
Least-squares SGD is the least-mean-squares (LMS) recursion. Its $L^p$ stability theory \citep{guo1997}, Markov-operator perturbations \citep{delyon1999}, and tracking-error expansions \citep{aguech2000} provide routes to moment control. Markovian analyses cover bounded coefficients \citep{srikant2019,huo2023} and unbounded coefficients under super-Lyapunov drift and growth conditions \citep{durmus2021markov}. Our positive-semidefinite (PSD) product bound establishes pathwise contraction from design integrability. Appendix~\ref{app:comparison} compares the temporal conditions, including the complementary martingale-design setting of \citet{shen2026}. Our causal framework uses functional dependence \citep{wu2005} and projective martingale approximation \citep{wuwoodroofe2004,wu2007}.

\paragraph{Heavy tails and central probabilities.}
Affine-recursion theory establishes stationary existence and power-law tails \citep{brandt1986,kesten1973,goldie1991}, including Markov-dependent coefficients \citep{roitershtein2007}. Multiplicative SGD noise can generate heavy stationary tails from light-tailed inputs \citep{gurbuzbalaban2021,hodgkinson2021}; \citet{damek2026} refine the iid matrix-recursion justification using irreducible-proximal theory. Even bounded iid coefficients can obstruct exponential concentration \citep{durmus2021tails}. Here all positive stationary moments fail at every small fixed step, while the same laws admit a sharp Gaussian approximation in TV. Table~\ref{tab:literature} summarizes the distinctions.

\paragraph{Route through the results.}
The design assumption gives pathwise contraction; a logarithmic score moment adds stationary existence (Theorem~\ref{thm:stability}). An $L^2$ projective score condition yields Gaussian and Ornstein--Uhlenbeck (OU) limits (Theorem~\ref{thm:clt}). For iid centered errors independent of the design, finite variance gives conditional weak convergence, strengthened to TV for Gaussian errors (Theorems~\ref{thm:gaussian} and~\ref{up:thm:quenched}). With Gaussian errors, a fourth design moment identifies the conditional first-order correction; quantitative Gram dependence permits its unconditional sharp expansion (Theorems~\ref{up:thm:first} and~\ref{up:thm:rate}). Section~\ref{sec:experiments} tests the expansion and its coverage consequence, then studies long-run covariance calibration for dependent scores.

\section{Model and assumptions}\label{sec:assumptions}
Let $(\varepsilon_i)_{i\in\mathbb Z}$ be iid random elements, and let
$(X_i,Z_i)=G(\ldots,\varepsilon_{i-1},\varepsilon_i)$ be measurable. Set
$\F_i=\sigma(\varepsilon_j:j\le i)$ and
$\Proj_iU=\E(U\mid\F_i)-\E(U\mid\F_{i-1})$. The dimension $p$ is fixed.
For vectors and matrices, $\norm{\cdot}$ denotes Euclidean and operator norm, respectively; $\norm{A}_{\mathrm F}=(\tr A\trans A)^{1/2}$ is Frobenius norm. Write $\norm{U}_q=(\E\norm{U}^q)^{1/q}$ and $U\in L^q$ when this is finite. We use $d_{\TV}(\mu,\nu)=\sup_{A\in\mathcal B(\R^p)}|\mu(A)-\nu(A)|$, where $\mathcal B(\R^p)$ denotes the Borel sets; this equals $\frac12\int|f-g|$ for densities $f,g$. The standard Gaussian density and CDF are $\phi$ and $\Phi$, with quantile $z_u=\Phi^{-1}(u)$; $\phi_p$ is the $N_p(0,I)$ density. Write
\begin{equation}\label{eq:error}
 e_i=\beta_i-\beta^*,\qquad H_i=X_iX_i\trans,\qquad H=\E H_0,\qquad
 \xi_i=X_iZ_i,\qquad e_i=(I-aH_i)e_{i-1}+a\xi_i.
\end{equation}

\begin{assumption}[Design]\label{ass:D}
$\E\norm{X_0}^2<\infty$ and $H\succ0$.
\end{assumption}
\begin{assumption}[Score]\label{ass:S}
$\E\xi_0=0$, $\xi_0\in L^2$, and
\begin{equation}\label{eq:hannan}
 \Theta_2=\sum_{k=0}^\infty\norm{\Proj_0\xi_k}_2<\infty.
\end{equation}
\end{assumption}

Assumption~\ref{ass:S} is a projective short-range dependence condition. For any causal process $U_k$, let $U_k^{\{0\}}$ replace $\varepsilon_0$ by an independent copy, and define
\begin{equation}\label{eq:fd}
 \delta_{U,q}(k)=\norm{U_k-U_k^{\{0\}}}_q,\qquad k\ge0.
\end{equation}
This notation applies to $U=X,H,\xi$. The sufficient condition $\sum_{k\ge0}\delta_{\xi,2}(k)<\infty$ implies \eqref{eq:hannan}.
Square-integrable martingale-difference scores also satisfy \eqref{eq:hannan}, since $\Proj_0\xi_k=0$ for $k\ge1$. In particular, independent centered errors with finite variance satisfy the score condition for every causal design obeying Assumption~\ref{ass:D}. The design itself may have arbitrarily slow ergodic averaging. Neither a bounded conditional design covariance nor a moment of the iterate is assumed.

Define backward products in their actual multiplication order:
\begin{equation}\label{eq:products}
 P_0(a)=I,\qquad P_j(a)=(I-aH_0)(I-aH_{-1})\cdots(I-aH_{-j+1}).
\end{equation}
Write $(\beta_i^{(a)},e_i^{(a)})$ for the jointly stationary solution at step $a$, when it exists, and $\pi_a$ for its marginal iterate law. Unadorned iterates with a specified $\beta_0$ denote an initialized recursion. Stationarity and uniqueness concern the iterate coupled to the data process; the iterate alone need not be Markov.

\section{Gaussian marginals with no stationary moments}\label{sec:refresh}
Let $B_i\sim\mathrm{Bernoulli}(\rho)$ and $G_i\sim N(0,1)$ be independent iid sequences. The stationary design
\begin{equation}\label{up:eq:refresh}
X_i=B_iX_{i-1}+(1-B_i)G_i
\end{equation}
equals the most recent refreshed Gaussian value. Its marginal law is $N(0,1)$ for every $\rho$, and its transition kernel satisfies
$P^n(x,\cdot)=\rho^n\delta_x+(1-\rho^n)N(0,1)$. Hence its uniform total-variation mixing distance is exactly $\rho^n$.

This construction isolates whether temporal persistence alone can destroy stationary moments while preserving a Gaussian limit.

\begin{theorem}[Same Gaussian marginals, different stationary moments]\label{up:thm:refresh}
For the design~\eqref{up:eq:refresh} with $0<\rho<1$ and iid nondegenerate centered errors of variance $\sigma^2<\infty$, independent of the entire design, the stationary error exists throughout the sufficient range $0<a<2/3$ and
\begin{equation}\label{up:eq:refresh-mom}
\E|e_0^{(a)}|^r=\infty\quad(r>0),\qquad
e_0^{(a)}/\sqrt a\Rightarrow N(0,\sigma^2/2)\quad(a\downarrow0).
\end{equation}
For Gaussian errors, the convergence is in total variation, also conditional on almost every design path. In contrast, at $\rho=0$,
\begin{equation}\label{up:eq:iid-var}
\E(e_0^{(a)})^2=a\sigma^2/(2-3a)<\infty\qquad(0<a<2/3).
\end{equation}
For each finite $q\ge1$, $\delta_{X,q}(k)\le2\|G_0\|_q\rho^{k/q}$ and $\delta_{H,q}(k)\le2\|G_0^2\|_q\rho^{k/q}$.
\end{theorem}

A single extreme design value can survive across many updates, repeatedly amplifying an earlier regression error. For $|X|\ge u$ with $au^2>2$, each retained value multiplies its coefficient by at least $au^2-1$. A run of length $m$ has probability proportional to $\rho^m$. For any $r>0$, choosing $u$ so that $\rho(au^2-1)^r>1$ makes the $r$th-moment amplification dominate this probability decay. Appendix~\ref{app:refresh} formalizes the argument by isolating one error coefficient.

With Gaussian errors, each observation $(X_i,Z_i,Y_i)$ is jointly Gaussian, but the \emph{time series} is not jointly Gaussian. With Rademacher errors, observations are marginally sub-Gaussian and the score has an exactly Gaussian marginal; every stationary moment still fails. Appendix~\ref{app:refresh} proves these assertions and a subpolynomial tail lower bound. A second construction uses the positive AR(1) design $X_i=\rho X_{i-1}+\eta_i$, $0<\rho<1$. Pareto-$3$ innovations give a finite second design moment, while $\eta_i=\exp(cE_i^{2/3})$, $E_i\sim\mathrm{Exp}(1)$, gives moments of every order. Both yield failure of all stationary iterate moments. These examples, proved in Appendix~\ref{sec:examples}, show how the qualitative and quantitative results apply to ordinary linear-filter designs as well as refresh processes.

\section{Stationary limits without iterate moments}\label{sec:results}
The examples call for stability at the level of individual data paths. We establish contraction of the update products first, so stationary existence does not require an iterate $L^r$ bound.
\begin{theorem}[Pathwise stability and stationary existence]\label{thm:stability}
Under Assumption~\ref{ass:D}, there exist deterministic $a_0,c>0$ and a random $K<\infty$ almost surely such that, on a common probability-one event,
\begin{equation}\label{eq:uniformpath}
 \norm{P_j(a)}\le K e^{-caj}\quad\text{for all }j\ge0,\ 0<a\le a_0.
\end{equation}
If $\E\log(1+\norm{\xi_0})<\infty$, the recursion has a unique jointly stationary solution, which is causal and is represented by
\begin{equation}\label{eq:perpetuity}
 e_0^{(a)}=a\sum_{j=0}^\infty P_j(a)\xi_{-j}.
\end{equation}
The series converges absolutely almost surely. For a fixed $a\le a_0$, any finite initialization is forgotten almost surely when coupled to this solution on the same data.
\end{theorem}

The estimate \eqref{eq:uniformpath} is derived from the design distribution. It gives pathwise stability even when $\E K^r=\infty$ for every $r>0$, as the examples in Appendix~\ref{sec:examples} illustrate.

Once the products contract, score cancellation determines the central fluctuation. Under Assumption~\ref{ass:S}, the long-run covariance
\begin{equation}\label{eq:omega}
 \Omega=\sum_{k\in\mathbb Z}\E(\xi_0\xi_k\trans)
\end{equation}
exists with absolute entrywise convergence. Let $V$ be the unique positive semidefinite solution of
\begin{equation}\label{eq:lyapunov}
 HV+VH=\Omega,\qquad
 V=\int_0^\infty e^{-tH}\Omega e^{-tH}\,dt.
\end{equation}

\begin{theorem}[Stationary central limit theorem]\label{thm:clt}
Under Assumptions~\ref{ass:D}--\ref{ass:S}, for $B_a\sim\pi_a$,
\begin{equation}\label{eq:clt}
 a^{-1/2}(B_a-\beta^*)\Rightarrow N_p(0,V),\qquad a\downarrow0.
\end{equation}
For every finite $T$, the stationary process
$U_a(t)=a^{-1/2}(\beta^{(a)}_{\lfloor t/a\rfloor}-\beta^*)$ converges in
$D([0,T],\R^p)$, the space of right-continuous paths with left limits equipped with the Skorokhod $J_1$ topology, to the stationary solution of
\begin{equation}\label{eq:ou}
 dU_t=-HU_t\,dt+\Omega^{1/2}\,dW_t.
\end{equation}
Here $W$ is standard $p$-dimensional Brownian motion. Both limits hold without any positive moment of $B_a$.
\end{theorem}

The centering in \eqref{eq:clt} is $\beta^*$, even if $\int\beta\,\pi_a(d\beta)$ is undefined. Likewise, $V$ is the covariance of the limiting Gaussian distribution and need not be a limit of stationary covariance matrices.

\subsection{An exact identity for Gaussian regression errors}
Independent Gaussian errors turn the stationary series into an exact Gaussian law conditional on the design. Thus distributional approximation reduces to controlling a random covariance perturbation. Put
$\G_X=\sigma(X_i:i\in\mathbb Z)$.

\begin{theorem}[Conditional total-variation limit]\label{thm:gaussian}
Suppose Assumption~\ref{ass:D} holds and $(Z_i)$ is iid $N(0,\sigma^2)$, $\sigma>0$, independent of $\G_X$. For all sufficiently small $a$, define
\begin{equation}\label{eq:R}
 \mathcal R_a=a^2\sum_{j=0}^\infty P_j(a)H_{-j}^2P_j(a)\trans.
\end{equation}
This matrix is finite and positive semidefinite almost surely, and
\begin{equation}\label{eq:conditional}
 \Law\bigl(a^{-1/2}e_0^{(a)}\mid\G_X\bigr)
 =N_p\left(0,\frac{\sigma^2}{2}(I+\mathcal R_a)\right),
 \qquad \norm{\mathcal R_a}\longrightarrow0\quad\as
\end{equation}
Consequently, conditional and unconditional convergence to
$N_p(0,\sigma^2 I/2)$ hold in total variation. Quantitatively, the conditional distance is bounded by
\begin{equation}\label{eq:tvbound}
 d_{\TV}\le \min\{1,\norm{\mathcal R_a}_{\mathrm F}/(2\sqrt2)\}.
\end{equation}
\end{theorem}

The covariance identity follows by telescoping $P_jP_j\trans-P_{j+1}P_{j+1}\trans$ and requires only $\E\norm{H_0}<\infty$. The perturbation vanishes along almost every design path. Bounded convergence then applies to \eqref{eq:tvbound} even when $\E\norm{\mathcal R_a}^r=\infty$. For the population noise variance, Gaussian central spheres have stationary coverage at most their nominal level and converge to that level; Appendix~\ref{app:gaussian} gives the argument.

\subsection{Conditional limits for non-Gaussian errors}
For independent errors with finite variance, the same conditional weights still yield a Gaussian weak limit. This extends the qualitative conclusion to continuous and discrete error laws.

\begin{theorem}[Quenched CLT under finite error variance]\label{up:thm:quenched}
Under Assumption~\ref{ass:D} and arbitrary iid centered errors with $\sigma^2=\E Z_0^2<\infty$, independent of the entire design,
\[
 \Law(e_0^{(a)}/\sqrt a\mid\mathcal G_X)
 \Rightarrow N_p(0,\sigma^2I/2)
 \quad\text{for almost every design realization}.
\]
The conclusion holds under the qualitative design assumption alone.
\end{theorem}

The proof combines the Gaussian identity's sum of squared weights with a vanishing maximal weight and the independent-array Lindeberg theorem.

\section{A sharp first-order Gaussian approximation}\label{sec:rate}
The exact conditional law also identifies the leading approximation error. A fourth design moment determines its coefficient: when $\E\norm{H_0}^2<\infty$, define
\begin{equation}\label{eq:Jmain}
J=\int_0^\infty e^{-tH}\E(H_0^2)e^{-tH}\,dt,
\qquad HJ+JH=\E(H_0^2).
\end{equation}
The matrix $J$ corrects the covariance of the approximating Gaussian law. The stationary covariance itself may be infinite or undefined.

\begin{theorem}[Conditional first-order expansion]\label{up:thm:first}
Retain Assumption~\ref{ass:D} and additionally assume $\E\norm{H_0}^2<\infty$. With $J$ from \eqref{eq:Jmain},
\begin{equation}\label{up:eq:R-first}
 \mathcal R_a/a\longrightarrow J\quad\text{almost surely}.
\end{equation}
Under the Gaussian-error assumptions of Theorem~\ref{thm:gaussian}, this implies conditional $O(a)$ total-variation error relative to $N_p(0,\sigma^2I/2)$, and conditional $o(a)$ error relative to $N_p(0,\sigma^2(I+aJ)/2)$, almost surely.
\end{theorem}

Averaging the conditional expansion requires control of rare design environments at the $a$ scale. For this unconditional rate, impose
\begin{equation}\label{up:eq:rate-ass}
H_0\in L^q,\qquad\sum_{k\ge0}\delta_{H,q}(k)<\infty
\quad\text{for some }q>2.
\end{equation}
It follows, for example, from $X_0\in L^{2q}$ and $\sum_k\delta_{X,2q}(k)<\infty$. The qualitative limits and conditional expansion retain their weaker assumptions.

\begin{theorem}[Unconditional first-order Gaussian approximation]\label{up:thm:rate}
Assume Assumption~\ref{ass:D},~\eqref{up:eq:rate-ass}, and the Gaussian-error assumptions of Theorem~\ref{thm:gaussian}. Let
\[
 \mu_a=\Law(e_0^{(a)}/\sqrt a),\quad
 \gamma_0=N_p(0,\sigma^2I/2),\quad
 \gamma_a=N_p(0,\sigma^2(I+aJ)/2),
\]
where $J$ is defined in \eqref{eq:Jmain}. There is a deterministic finite $C$, depending only on the design law and $p$, such that, for all sufficiently small positive $a$,
\begin{align}
 d_{\TV}(\mu_a,\gamma_0)&\le Ca,\label{up:eq:TVrate}\\
 d_{\TV}(\mu_a,\gamma_a)&=o(a),\label{up:eq:TVcorrected}\\
 \lim_{a\downarrow0}\frac{d_{\TV}(\mu_a,\gamma_0)}{a}
 &=\frac14\E\left|G^\top JG-\tr J\right|,
 \qquad G\sim N_p(0,I).\label{up:eq:TVconstant}
\end{align}
The constant in~\eqref{up:eq:TVconstant} is positive. None of these conclusions assumes any stationary iterate moment. The deterministic correction removes the leading order-$a$ error.
\end{theorem}

The refresh design now brings the moment obstruction and the sharp expansion together.

\begin{corollary}[Sharp Gaussian approximation with no stationary moments]\label{cor:quantitative-coexistence}
For the refresh design with fixed $0<\rho<1$ and independent $N(0,\sigma^2)$ errors, let $\mu_a=\Law(e_0^{(a)}/\sqrt a)$. Then
\begin{align*}
\E|e_0^{(a)}|^r&=\infty &&(r>0,\ 0<a<2/3),\\
d_{\TV}(\mu_a,N(0,\sigma^2/2))&=\tfrac32\phi(1)a+o(a),\\
d_{\TV}(\mu_a,N(0,\sigma^2(1+3a/2)/2))&=o(a) &&(a\downarrow0).
\end{align*}
\end{corollary}
Indeed, Gaussian marginal moments and the dependence bounds in Theorem~\ref{up:thm:refresh} verify \eqref{up:eq:rate-ass}, and $J=3/2$. The log-Weibull AR example also satisfies the quantitative assumptions; the Pareto-$3$ example illustrates the weaker qualitative theory. Each conclusion holds for a fixed dependence law.

The expansion also predicts the leading coverage error of a Gaussian interval.
\begin{corollary}[First-order scalar coverage]\label{cor:coverage-first}
Under Theorem~\ref{up:thm:rate} with $p=1$, let $\alpha\in(0,1)$ be fixed and put $z=z_{1-\alpha/2}$ and $h_a=\sigma\sqrt{a/2}\,z$. Then
\begin{align}
\Prob(|e_0^{(a)}|\le h_a)&=1-\alpha-aJz\phi(z)+o(a),\label{eq:coverage-first}\\
\Prob(|e_0^{(a)}|\le h_a\sqrt{1+aJ})&=1-\alpha+o(a).\label{eq:coverage-corrected}
\end{align}
\end{corollary}
The deterministic correction removes the leading undercoverage. These intervals use the population noise variance; Appendix~\ref{app:coverage-first} proves the expansion directly from the TV bound.

\subsection{Initialization and central confidence regions}
To use the stationary theory after a finite run, initialization must be negligible on the approximation's scale. Weak convergence uses the $\sqrt a$ scale; preserving the first-order expansion requires a sharper bound.
\begin{corollary}[Last iterate after initialization]\label{cor:startup}
Under Assumptions~\ref{ass:D}--\ref{ass:S}, initialize at a fixed finite $\beta_0$. If
\begin{equation}\label{eq:burnin}
 a^{-1/2}e^{-can(a)}\longrightarrow0,
\end{equation}
with $c$ from Theorem~\ref{thm:stability}, then
$a^{-1/2}(\beta_{n(a)}-\beta^*)\Rightarrow N_p(0,V)$.
One sufficient condition is $an(a)/\log(1/a)\to\infty$.
\end{corollary}

\begin{corollary}[Finite-history first-order approximation]\label{cor:startup-tv}
Under the assumptions of Theorem~\ref{up:thm:rate}, initialize with deterministic error $b=\beta_0-\beta^*$. There are $a_0,c,C,t_0>0$ and a nonnegative $L_*$ with $\E L_*<\infty$ such that, for $0<a\le a_0$ and $an\ge t_0$,
\begin{equation}\label{eq:startup-tv}
d_{\TV}\!\left(\Law(e_n/\sqrt a),\mu_a\right)
\le C\bigl(\norm b\,a^{-1/2}e^{-can}+e^{-2can}\bigr)+\Prob(L_*>1/a).
\end{equation}
If $e^{-can(a)}=o(a^{3/2})$, the initialized law $\mu_{a,n}=\Law(e_{n(a)}/\sqrt a)$ satisfies
\[
d_{\TV}(\mu_{a,n},\gamma_a)=o(a),\qquad
\lim_{a\downarrow0}\frac{d_{\TV}(\mu_{a,n},\gamma_0)}a
=\frac14\E|G\trans JG-\tr J|,
\]
with $\gamma_a,\gamma_0,G$ as in Theorem~\ref{up:thm:rate}. For $b=0$, it suffices that $e^{-2can(a)}=o(a)$. The schedule $n(a)=\lceil\max\{24,2\log^2(1/a)\}/a\rceil$ satisfies both conditions for each fixed admissible law.
\end{corollary}
The transfer compares finite and stationary conditional Gaussian laws at the same endpoint (Appendix~\ref{app:finite-tv}). Under its stated history conditions, Corollary~\ref{cor:coverage-first} also holds with $e_{n(a)}$ in place of $e_0^{(a)}$.

Under independent centered errors with variance $\sigma^2\in(0,\infty)$, the score is a martingale difference and $\Omega=\sigma^2H$, so $V=\sigma^2 I/2$ even for non-Gaussian errors. Let $\widehat\sigma_m^2$ be the ordinary least-squares residual variance from $m\to\infty$ observations. Ergodicity and second moments imply $\widehat\sigma_m^2\to\sigma^2$ in probability. Therefore, for a stationary iterate or an iterate satisfying \eqref{eq:burnin},
\begin{equation}\label{eq:confidence}
 \mathcal C_a=\left\{b\in\R^p:
 \frac{2\norm{\beta_{n(a)}-b}^2}{a\widehat\sigma_m^2}
 \le\chi^2_{p,1-\alpha}\right\}
 \quad\text{satisfies}\quad
 \Prob(\beta^*\in\mathcal C_a)\longrightarrow1-\alpha.
\end{equation}
Here $\chi^2_{p,u}$ is the $u$-quantile of the chi-square law with $p$ degrees of freedom. Calibration data may overlap the SGD data. Variance consistency gives asymptotic marginal coverage; a first-order expansion would additionally require control of calibration error at the $a$ scale. Appendix~\ref{app:extensions} proves this result and extends it to fixed positive definite preconditioning and fixed nonoverlapping mini-batches.

\paragraph{Calibration before burn-in.}
At a fixed earlier endpoint, Gaussian conditioning gives a complementary exact calibration. Propagating the initialization and noise coefficients yields a Gaussian pivot; a full-rank residual-variance block with disjoint error indices gives an exact $F$ pivot. Appendix~\ref{app:exact} states the rank conditions and formulas, allowing dependent designs across blocks. The experiment reports both coverage and interval width to show the precision available before burn-in.

\section{Proof mechanism: pathwise contraction and cancellation}\label{sec:proof}
Positive-semidefinite update products contract pathwise, while score partial sums satisfy an invariance principle. Over $O(1/a)$ updates, or $O(1)$ learning time, the products approach $e^{-tH}$. Localization then removes the remote stationary past. Appendices~\ref{app:stability}--\ref{app:clt} supply the proofs.

\paragraph{A truncated block-product inequality.}
For positive semidefinite $B_1,\ldots,B_m$, let $s=\sum_\ell\norm{B_\ell}$ and
$g=\lambda_{\min}(\sum_\ell B_\ell)$. Direct expansion gives
\begin{equation}\label{eq:block}
 \log\norm{(I-aB_m)\cdots(I-aB_1)}
 \le-ag+2as\min(as,1).
\end{equation}
For $as\le1$, the remainder about $I-a\sum B_\ell$ is at most $2a^2s^2$; for $as>1$, bound the log norm by $as$. The right side has finite expectation after division by $a$ under only $\E s<\infty$. Ergodicity and $H\succ0$ provide a fixed block length with $\mu=\E g>0$. Choose $a_0$ so that
$\E[2s\min(a_0s,1)]<\mu/2$. Block log norms are then bounded by $a$ times a stationary integrable sequence of negative mean, simultaneously for $a\le a_0$. The ergodic theorem proves \eqref{eq:uniformpath}.

\paragraph{Finite learning-time windows.}
Put $S_n=\sum_{j=0}^{n-1}\xi_{-j}$. Assumption~\ref{ass:S} gives
\begin{equation}\label{eq:scoremax}
 \sqrt a\,S_{\lfloor t/a\rfloor}\Rightarrow\Omega^{1/2}W_t,
 \qquad \E\max_{j\le n}\norm{S_j}^2\le Cn.
\end{equation}
The ergodic theorem gives uniform small-step Cesaro convergence of the centered Gram sums. Abel summation in the matrix Euler equation then yields $P_{\lfloor t/a\rfloor}(a)\to e^{-tH}$ uniformly for $t\le T$, directly for the original unbounded coefficients.
The products also have bounded total variation in probability. Summation by parts with \eqref{eq:scoremax} yields
\begin{equation}\label{eq:recent}
 \sqrt a\sum_{j<\lceil T/a\rceil}P_j(a)\xi_{-j}
 \Rightarrow\int_0^T e^{-tH}\Omega^{1/2}\,dW_t.
\end{equation}
This argument accommodates dependence between the products and scores.

\paragraph{The infinite stationary past.}
The essential interchange of limits is established separately. Set $N=\lceil T/a\rceil$ and
$R_{a,T}=\sqrt a\sum_{j\ge N}P_j(a)\xi_{-j}$. Abel summation gives
\begin{equation}\label{eq:abeltail}
 R_{a,T}=-\sqrt a P_N(a)S_N
     +a^{3/2}\sum_{j\ge N}P_j(a)H_{-j}S_{j+1}.
\end{equation}
The random variable $D=\sup_{n\ge1}n^{-1}\sum_{j<n}\norm{H_{-j}}$ is finite almost surely. On $\{K\le M,D\le M\}$, dyadic decomposition and \eqref{eq:scoremax} imply
\begin{equation}\label{eq:localization}
 \Prob(\norm{R_{a,T}}>\epsilon)
 \le\Prob(K\vee D>M)+\frac{CM^2}{\epsilon}(1+T)^{3/2}e^{-c'T}.
\end{equation}
Only score maxima are integrated. Taking $T\to\infty$ and then $M\to\infty$ proves that the remote past is negligible uniformly as $a\downarrow0$. Combining this with \eqref{eq:recent} gives Theorem~\ref{thm:clt}. Proposition~\ref{prop:bl} expresses this localization as an explicit bound for bounded test functions.

\paragraph{Averaging the first-order correction.}
For the sharp TV expansion, the product bound becomes $\|P_j(a)\|\le e^{aL_*}e^{-caj}$ with $\E L_*<\infty$. On $\{aL_*\le1\}$, weighted Gram squares provide an integrable bound; the complementary event contributes $\Prob(L_*>1/a)=o(a)$. This localization averages the conditional expansion without requiring $\E\mathcal R_a<\infty$ (Appendix~\ref{app:rate}).

\section{Experiments}\label{sec:experiments}
We first test the sharp distributional expansion and its coverage consequence in the refresh design. A dependent-score experiment then examines long-run covariance calibration under the general Gaussian limit. Every configuration uses independent replicated histories; methods share paths within a configuration. Gaussian scalar experiments integrate regression errors conditional on the design. Actual SGD is used for non-Gaussian errors and finite-history inference. Appendix~\ref{app:experiments} provides all protocols and Monte Carlo uncertainty.

\paragraph{Direct distributional error and the first-order correction.}
The refresh experiment holds the $N(0,1)$ design marginal fixed and varies $\rho\in\{0,.7,.9\}$. Five steps range from $.04$ to $.0025$, with 24,000 histories per configuration and $n=\lceil\max\{24,2\log^2(1/a)\}/a\rceil$ updates. Corollary~\ref{cor:startup-tv} shows that this schedule preserves the stationary first-order expansion asymptotically. We estimate TV by integrating the absolute density difference between the empirical conditional Gaussian mixture and its reference Gaussian. Figure~\ref{fig:refresh}(a) compares both approximations on a common zero-based scale: the baseline approaches the predicted constant $3\phi(1)/2$, and correction reduces the normalized error. At $\rho=.7,a=.0025$, estimated TV changes from $8.87\times10^{-4}$ to $1.85\times10^{-5}$. Panel (b) shows concentration of the finite-history correction $\mathcal R_{a,n}/a$ toward $J=1.5$. Panel (c) connects this coefficient to coverage. At the same step, Corollary~\ref{cor:coverage-first} predicts approximately $0.043$ percentage points of baseline undercoverage; the observed deficit is $0.0421$ points. Correction raises coverage from $94.9579\%$ to $95.0009\%$, a $0.0430$-point increase. The leading term thus predicts the observed scale of the coverage change. Table~\ref{tab:refresh} reports absolute errors and uncertainty.

\begin{figure}[!ht]
\centering\includegraphics[width=\linewidth]{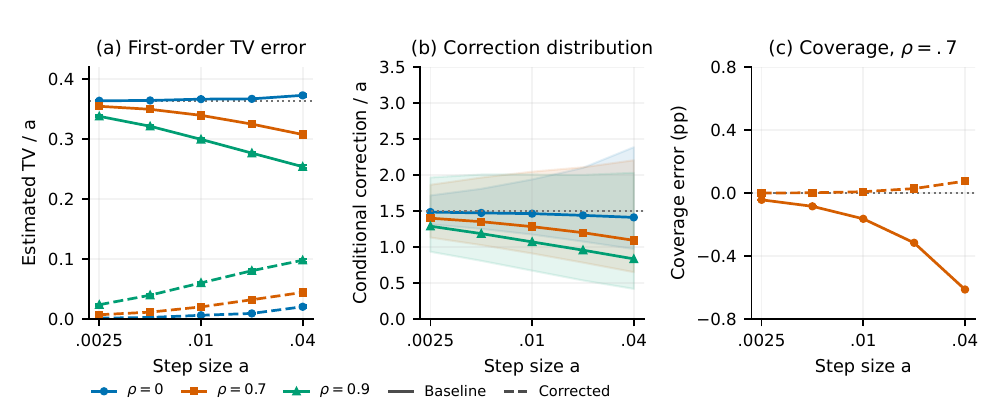}
\caption{The first-order approximation in the Gaussian-refresh design. (a) Estimated TV divided by $a$ for baseline (solid) and corrected (dashed) laws; the dotted line is $3\phi(1)/2$. Bars are 1.96 linearized Monte Carlo standard errors. (b) Median and 10th--90th percentiles of the finite-history correction $\mathcal R_{a,n}/a$; the dotted line is $J=1.5$, and shading represents variation across histories. (c) Coverage minus 95\%, in percentage points (pp), with 1.96 standard-error bars; Corollary~\ref{cor:coverage-first} predicts baseline error $-100aJz_{.975}\phi(z_{.975})+o(a)$ pp. Each configuration uses 24,000 histories.}\label{fig:refresh}
\end{figure}

\paragraph{Feasible inference with a dependent score.}
Theorem~\ref{thm:clt} also covers serially dependent scores, whose long-run covariance determines the Gaussian limit. We use a noncommuting two-dimensional Gaussian VAR(1) design. Gaussian AR(1) errors, independent of the design, have coefficient $.7$, giving $\Omega\ne H$. With 6,000 replications, we compare oracle Lyapunov calibration, lag-zero calibration, and residual-score Bartlett estimation using an earlier block on the same stream. This provides heteroskedasticity and autocorrelation consistent (HAC) calibration. At $a=.0025$, oracle coverage is 95.23\%, lag-zero coverage is 76.17\%, and HAC coverage reaches 94.38\% with 8,192 calibration observations and bandwidth 19. The lag-zero Gaussian limiting coverage is 76.01\%. A population Bartlett benchmark isolates the effect of bandwidth: at bandwidth 19 its limiting coverage is 94.37\%, close to the observed HAC result (Monte Carlo standard error $0.30$ pp). Figure~\ref{fig:hac}(b) compares this benchmark with feasible estimates. The agreement supports finite-bandwidth bias as an explanation of the remaining gap; larger bandwidths capture more temporal covariance in this experiment.

\begin{figure}[!ht]
\centering\includegraphics[width=\linewidth]{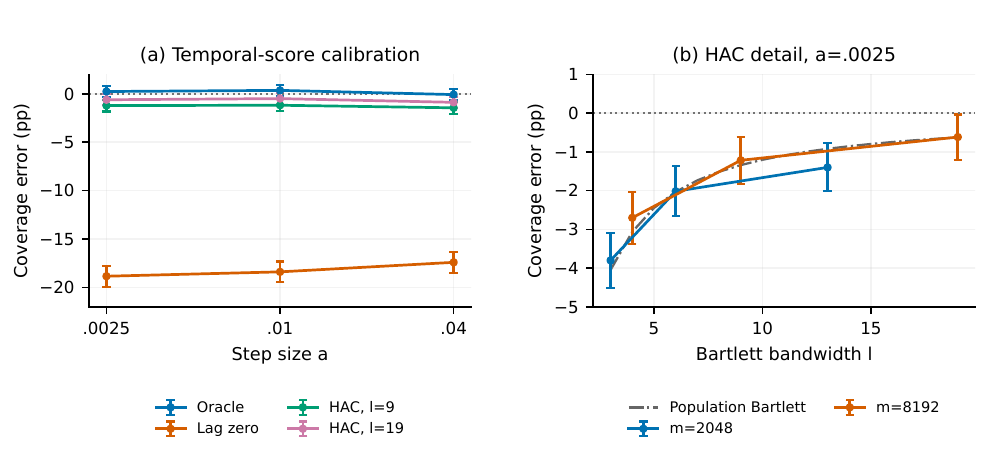}
\caption{Temporal covariance matters for last-iterate inference. Coverage errors are in percentage points relative to 95\%, with 1.96 Monte Carlo standard-error bars from 6,000 replications. (a) Oracle, lag-zero and feasible HAC calibration; HAC uses $m=8192$. (b) HAC bandwidth sensitivity at $a=.0025$, with fixed axis limits $[-5,1]$ pp. The gray dash-dotted curve is the population Bartlett Gaussian-limit benchmark; feasible estimates use observed calibration data.}\label{fig:hac}
\end{figure}

\paragraph{Scope and initialization.}
Pareto AR designs, finite-variance non-Gaussian errors and bounded two-time characteristic functions support the qualitative stationary and OU predictions (Figure~\ref{fig:additional}). The log-Weibull comparison shows that entry into the first-order regime depends on the design law. Correction is accurate at innovation scale $c=.5$; at $c=1$, the baseline remains closer to nominal over the tested steps (Tables~\ref{tab:other} and~\ref{tab:stress}). Actual Gaussian SGD with nonzero initialization and a separate residual-variance calibration block attains near-nominal history-corrected coverage from early iterations; Figure~\ref{fig:startupshock} reports the accompanying interval widths. At $a=.01$, the median estimated-variance interval width decreases from $2.186$ at $an=.05$ to $.283$ at $an=24$; the corresponding 90th percentiles are $6.242$ and $.316$. These widths quantify the precision cost of early inference alongside its coverage.

\section{Discussion}
Temporal persistence can destroy every stationary moment while preserving a sharp Gaussian approximation. Pathwise contraction and score cancellation separate stability along typical paths from integrability across environments. Conditional Gaussian structure turns this distinction into a quantitative law approximation and a leading coverage correction.

For fixed-dimensional least-squares SGD, these results support central inference even when stationary means and covariances are unavailable. Finite-history transfer connects the expansion to initialized runs; exact Gaussian calibration also permits earlier inference, with interval widths recording precision. With dependent scores, calibration instead uses their long-run covariance. The log-Weibull results show that entry into the first-order regime depends on the design law. Extending quantitative approximation to general dependent scores remains an open direction.
\label{main:end}

\clearpage

\section*{Reproducibility statement}
Appendices provide complete proofs and the numerical protocol. The arXiv ancillary files contain the simulation and analysis scripts, declared configurations and seeds, saved analysis summaries, dependency versions, and algebraic checks. The experiments use synthetic data. To keep this distribution compact, most replication-level arrays are omitted; the supplied seeds and code regenerate them and the synthetic streams. Full per-update trajectories were processed online rather than stored. The ancillary README gives commands to reproduce all figures and tables from the saved summaries and to regenerate the underlying data.

\section*{AI use statement}
In this work, we used generative AI tools to assist with manuscript writing and editing, literature search and identification, drafting portions of the manuscript, experimental design and code development, exploration of new research directions based on existing results, and the development and verification of mathematical arguments and proofs.

We have reviewed all AI-assisted work. Suggested references were checked against original sources, experimental code and numerical results were executed and verified by the authors, and mathematical claims, derivations, and proofs were independently checked for correctness.

We take responsibility for the final content of this work, including all text, claims, proofs, code, experiments, and other artifacts produced with the aid of generative AI.

\clearpage
\appendix

\section{Pathwise stability and stationary existence}\label{app:stability}
All almost sure statements below can be imposed on a common probability-one event. The underlying shift of an iid sequence is mixing in the ergodic-theoretic sense, and its nonzero integer powers are ergodic. Every measurable causal factor inherits these properties. This fact is used only to apply the ergodic theorem to fixed nonoverlapping blocks; it is not a quantitative mixing assumption.

\begin{lemma}[Integrable sequences and exponential weights]\label{lem:weights}
Let $(h_j)_{j\ge0}$ be a stationary nonnegative sequence with $\E h_0<\infty$. Almost surely,
\begin{equation}\label{eq:hproperties}
 D_h:=\sup_{n\ge1}\frac1n\sum_{j=0}^{n-1}h_j<\infty,
 \qquad \frac{h_j}{j+1}\longrightarrow0.
\end{equation}
For every $c>0$, there is a deterministic $C_c<\infty$ such that, for $0<a\le1$,
\begin{align}
 a\sum_{j\ge0}e^{-caj}h_j&\le C_cD_h,\label{eq:weighted1}\\
 a^2\sum_{j\ge0}(j+1)e^{-caj}h_j&\le C_cD_h.\label{eq:weighted2}
\end{align}
Furthermore,
\begin{equation}\label{eq:weighted-square}
 a^2\sum_{j\ge0}e^{-caj}h_j^2\longrightarrow0\quad\as
 \qquad(a\downarrow0).
\end{equation}
\end{lemma}
\begin{proof}
The first assertion of \eqref{eq:hproperties} follows from the ergodic theorem, with the possibly random invariant conditional expectation as limit. For every $\epsilon>0$,
$\sum_{j\ge1}\Prob(h_j>\epsilon j)=\sum_{j\ge1}\Prob(h_0>\epsilon j)<\infty$.
Borel--Cantelli proves the second assertion, first for rational $\epsilon>0$ and then for all $\epsilon>0$.
If $q=e^{-ca}$ and $C_n=\sum_{j<n}h_j$, Abel summation gives
\[
 \sum_{j\ge0}q^jh_j=(1-q)\sum_{j\ge0}q^j C_{j+1}
 \le\frac{D_h}{1-q}.
\]
The ratio $a/(1-e^{-ca})$ is bounded on $(0,1]$, proving \eqref{eq:weighted1}. Since $(j+1)e^{-caj/2}\le C_c/a$, equation~\eqref{eq:weighted2} follows from \eqref{eq:weighted1} with $c/2$.
Finally, given $\epsilon>0$, choose a random finite $J$ such that $h_j\le\epsilon(j+1)$ for $j\ge J$. The finite sum over $j<J$ in \eqref{eq:weighted-square} tends to zero, while the remainder is at most $\epsilon C_cD_h$ by \eqref{eq:weighted2}. Let $\epsilon\downarrow0$.
\end{proof}

\begin{lemma}[Deterministic block inequality]\label{lem:block}
For positive semidefinite matrices $B_1,\ldots,B_m$, write
$s=\sum_{\ell=1}^m\norm{B_\ell}$ and
$g=\lambda_{\min}(\sum_{\ell=1}^mB_\ell)$. For all $a>0$,
\[
 \log\norm{(I-aB_m)\cdots(I-aB_1)}\le-ag+2as\min(as,1),
\]
where $\log0=-\infty$.
\end{lemma}
\begin{proof}
Let $x=as$ and $Q=\prod_{\ell=m}^1(I-aB_\ell)$. Expansion in powers of $a$ does not require commutativity. The sum of norms of terms of order $r$ is bounded by $x^r/r!$. Consequently,
\[
 \norm{Q-\left(I-a\sum_\ell B_\ell\right)}\le e^x-1-x.
\]
For $x\le1$, every eigenvalue of $a\sum_\ell B_\ell$ belongs to $[0,1]$, so the norm of $I-a\sum_\ell B_\ell$ is $1-ag$. Since $e^x-1-x\le(e/2)x^2<2x^2$, the inequality $\log u\le u-1$ gives the result. For $x>1$,
\[
 \log\norm Q\le\sum_\ell\log(1+a\norm{B_\ell})\le as\le-ag+2as,
\]
using $g\le s$.
\end{proof}

\begin{proof}[Proof of Theorem~\ref{thm:stability}]
By the ergodic theorem,
$m^{-1}\sum_{j=0}^{m-1}H_{-j}\to H\succ0$ almost surely as $m\to\infty$.
Hence there exists a deterministic finite $m$ such that
\[
 \mu:=\E\lambda_{\min}\left(\sum_{j=0}^{m-1}H_{-j}\right)>0.
\]
This expectation is finite because the eigenvalue is bounded by the sum of matrix norms. Partition the past into blocks
$\{-(r-1)m,\ldots,-rm+1\}$, $r\ge1$, and let $g_r$ and $s_r$ be the corresponding smallest Gram eigenvalue and sum of norms. Dominated convergence gives
\[
 \E[2s_1\min(a_0s_1,1)]\longrightarrow0\quad(a_0\downarrow0),
\]
since $\E s_1<\infty$. Fix $a_0\le1$ such that this expectation is less than $\mu/2$, and set
\[
 z_r=-g_r+2s_r\min(a_0s_r,1),\qquad \eta=-\E z_1>0.
\]
Lemma~\ref{lem:block} bounds each block log norm by $az_r$, simultaneously for $0<a\le a_0$.
Ergodicity under the block shift implies $N^{-1}\sum_{r=1}^Nz_r\to-\eta$ almost surely. Also $s_{N+1}/N\to0$ almost surely, by the integrability and stationarity of block sums. Therefore
\[
 L=\sup_{N\ge0}\left\{\sum_{r=1}^Nz_r+s_{N+1}+\frac\eta2N\right\}<\infty\quad\as
\]
For $j=Nm+r$, $0\le r<m$, the incomplete block costs at most $as_{N+1}$ in log norm. Submultiplicativity gives
\[
 \log\norm{P_j(a)}
 \le aL-\frac{a\eta}2N
 \le a(L+\eta/2)-\frac{a\eta}{2m}j.
\]
Thus \eqref{eq:uniformpath} holds with
$c=\eta/(2m)$ and $K=\exp\{a_0(L+\eta/2)\}$. The same argument, applied to forward blocks, proves
\begin{equation}\label{eq:forward}
 \norm{(I-aH_n)\cdots(I-aH_1)}\le K^+e^{-can},
 \qquad 0<a\le a_0,
\end{equation}
for another almost surely finite random $K^+$. No integrability of $L$, $K$, or $K^+$ is used.

If $\E\log(1+\norm{\xi_0})<\infty$, stationarity and Borel--Cantelli imply
$j^{-1}\log(1+\norm{\xi_{-j}})\to0$ almost surely. Exponential contraction then implies absolute convergence of \eqref{eq:perpetuity} for every $a\in(0,a_0]$. Shifting the same series constructs a causal stationary process, and factoring out $I-aH_i$ verifies the recursion.

If $e_i$ and $\widetilde e_i$ are stationary solutions coupled to the same data, backward iteration gives
$e_0-\widetilde e_0=P_n(a)(e_{-n}-\widetilde e_{-n})$.
The differences on the right form a tight family, since the two marginal laws are stationary. The product tends to zero almost surely, hence the right side tends to zero in probability. The left side is therefore zero almost surely. This also proves uniqueness without a moment condition on a competing stationary solution. Finally, the difference from a recursion initialized at a finite value equals the forward product in \eqref{eq:forward} times its finite initial discrepancy. It tends to zero almost surely. In particular, the marginal law from that initialization converges weakly to $\pi_a$ for fixed $a$.
\end{proof}

\section{Score invariance principle and maximal inequality}\label{app:score}
For completeness, we record the precise consequences of the projective condition. They are standard martingale-approximation facts; compare \citet{wuwoodroofe2004,wu2007}.

\begin{lemma}[Projective score bounds]\label{lem:score}
Under Assumption~\ref{ass:S}, let $\theta_k=\norm{\Proj_0\xi_k}_2$ and
$D_i=\sum_{k\ge0}\Proj_i\xi_{i+k}$. Then $D_i$ is a stationary ergodic $L^2$ martingale-difference sequence, $\Omega=\E D_0D_0\trans$, and
\begin{equation}\label{eq:maxforward}
 \left\|\max_{1\le j\le n}\norm{\sum_{i=1}^j\xi_i}\right\|_2
 \le2\sqrt n\sum_{k\ge0}\theta_k.
\end{equation}
The backward maximum satisfies the same bound with constant $4$ instead of $2$. The partial-sum processes satisfy a functional CLT on every compact time interval, including jointly over finite intervals extending to negative times.
\end{lemma}
\begin{proof}
The iid past tail sigma-field is trivial. Causality and centering therefore give the $L^2$ expansion
$\xi_i=\sum_{k\ge0}\Proj_{i-k}\xi_i$.
For a fixed $k$, the sequence $(\Proj_{i-k}\xi_i)_i$ is a stationary martingale difference relative to $(\F_{i-k})_i$. Doob's inequality gives a maximal $L^2$ norm at most $2\sqrt n\theta_k$ for its partial sums. Minkowski's inequality and $\sum_k\theta_k<\infty$ prove \eqref{eq:maxforward}. By stationarity, a backward partial sum over a window of length $n$ has the distribution of a difference between the endpoint and a forward partial sum on that window. Its maximum is bounded by twice the forward maximum.

Orthogonality of innovation projections yields, for $k\ge0$,
\[
 \norm{\E\xi_0\xi_k\trans}
 \le\sum_{j\ge0}\theta_j\theta_{j+k}.
\]
Summing over positive and negative lags proves absolute covariance summability.
The series defining $D_i$ converges in $L^2$ by Minkowski's inequality; its terms are $\F_i$-measurable with zero conditional mean given $\F_{i-1}$. Hence $D_i$ is a stationary ergodic martingale difference.

Here are details of the functional approximation. For fixed $M$, define
\[
 \xi_i^{[M]}=\sum_{k=0}^M\Proj_{i-k}\xi_i,
 \qquad D_i^{[M]}=\sum_{k=0}^M\Proj_i\xi_{i+k}.
\]
Their partial sums differ by $R_0^{[M]}-R_n^{[M]}$, where
\[
 R_i^{[M]}=\sum_{k=1}^M\sum_{r=0}^{k-1}\Proj_{i-r}\xi_{i-r+k}
\]
is stationary and square integrable. Since $t^2\Prob(\norm{R_0^{[M]}}>t)\to0$,
\[
 \Prob\left(\max_{j\le n}\norm{R_j^{[M]}}>\epsilon\sqrt n\right)
 \le(n+1)\Prob(\norm{R_0^{[M]}}>\epsilon\sqrt n)\longrightarrow0.
\]
The stationary ergodic martingale $\sum_iD_i^{[M]}$ satisfies the functional CLT: its averaged conditional quadratic variation converges by ergodicity, and its conditional Lindeberg condition follows from the $L^2$ tail and Markov's inequality. Thus the same CLT holds for $\xi_i^{[M]}$.
The maximal approximation remainder satisfies
\[
 n^{-1/2}\left\|\max_{j\le n}\norm{\sum_{i=1}^j(\xi_i-\xi_i^{[M]})}\right\|_2
 \le2\sum_{k>M}\theta_k\longrightarrow0.
\]
Letting $M\to\infty$ proves the score functional CLT with covariance $\E D_0D_0\trans$. The analogous nonmaximal $L^2$ approximation shows that this covariance equals the limit of $n^{-1}\Cov(\sum_{i=1}^n\xi_i)$, which is $\Omega$ by absolute covariance summability.

Time reversal on a fixed window maps a forward partial-sum process to its endpoint minus its reversed path. Applied to its Brownian limit this gives a Brownian process with the same covariance. More generally, by stationarity one can begin a window at $-\lfloor L/a\rfloor$ and apply the forward functional CLT on a window of length of order $(L+T)/a$. Subtracting its value at time zero gives the asserted two-sided formulation and independent positive- and negative-time Brownian increments in the limit.
\end{proof}

To check \eqref{eq:fd}, observe that
$\Proj_0\xi_k=\E(\xi_k-\xi_k^{\{0\}}\mid\F_0)$.
Indeed, conditioning the coupled variable on $\F_0$ integrates out the replaced innovation and equals $\E(\xi_k\mid\F_{-1})$. Conditional Jensen's inequality implies $\theta_k\le\delta_{\xi,2}(k)$.

\section{Finite windows, remote past, and process convergence}\label{app:clt}
For a quantitative comparison that remains meaningful without moments, let $d_{\mathrm{BL}}$ be the supremum of expectation differences over functions bounded by one with Lipschitz constant at most one. Put $U_{a,T}=\sqrt a\sum_{j<\lceil T/a\rceil}P_j(a)\xi_{-j}$ and $V_T=\int_0^T e^{-tH}\Omega e^{-tH}\,dt$.
\begin{proposition}[A bounded-test-function approximation bound]\label{prop:bl}
Under Assumptions~\ref{ass:D}--\ref{ass:S}, for $a\le a_0$, $T\ge1$, and $M\ge1$,
\begin{align}\label{eq:blbound}
 d_{\mathrm{BL}}\bigl(\Law(a^{-1/2}e_0^{(a)}),N_p(0,V)\bigr)
 &\le d_{\mathrm{BL}}\bigl(\Law(U_{a,T}),N_p(0,V_T)\bigr)
       +2\Prob(K\vee D>M)\nonumber\\
 &\quad+CM^2(1+T)^{3/2}e^{-c'T}+Ce^{-\lambda_{\min}(H)T}.
\end{align}
\end{proposition}
The first term concerns a finite learning-time window; the remaining terms explicitly separate localization of the random environment from removal of the remote past. A rate in $a$ requires quantitative information about that finite-window approximation and the environment. The present assumptions ensure convergence, without assigning such a rate. The bound also explains why an approximation metric based on bounded events is appropriate when every Wasserstein distance of positive order is infinite.

\subsection{Averaging the matrix products}
\begin{lemma}[Uniform finite-window product limit]\label{lem:products}
Under Assumption~\ref{ass:D}, for every finite $T$,
\begin{equation}\label{eq:productlimit}
 \sup_{0\le t\le T}\norm{P_{\lfloor t/a\rfloor}(a)-e^{-tH}}\longrightarrow0
 \quad\as
\end{equation}
The total variation of $j\mapsto P_j(a)$ over $0\le j\le\lceil T/a\rceil$ is bounded almost surely, uniformly for sufficiently small $a$.
\end{lemma}
\begin{proof}
Write $h_j=\norm{H_{-j}}$ and $N=\lceil T/a\rceil$. By the ergodic theorem,
\[
 M_a=a\sum_{j<N}h_j=O(1)\quad\as,\qquad
 \delta_a:=\max_{k\le N}\norm{a\sum_{j<k}(H_{-j}-H)}\longrightarrow0\quad\as.
\]
Uniformity of the second statement follows directly from Cesaro convergence: after a fixed index the unscaled partial sums are at most $\epsilon k$, and the finite initial segment vanishes after multiplication by $a$.
For each $j\le N$,
\[
 \norm{P_j(a)}\le\prod_{\ell<j}(1+ah_\ell)\le e^{M_a},
 \quad
 \sum_{j<N}\norm{P_{j+1}(a)-P_j(a)}\le e^{M_a}M_a.
\]
Let $C_k=a\sum_{j<k}(H_{-j}-H)$. Discrete integration by parts yields
\[
 a\sum_{j<n}P_j(a)(H_{-j}-H)
 =P_{n-1}(a)C_n+\sum_{j=1}^{n-1}(P_{j-1}(a)-P_j(a))C_j.
\]
Its norm, uniformly for $n\le N$, is at most
$\delta_a e^{M_a}(1+M_a)=o(1)$ almost surely. The product recursion becomes
\[
 P_n(a)=I-a\sum_{j<n}P_j(a)H+o(1)
\]
uniformly over the window. Comparing with $(I-aH)^n$ and applying discrete Gronwall proves \eqref{eq:productlimit}. No coefficient truncation is needed for this finite-window argument.
\end{proof}

We use the following elementary continuity principle for the random weights. Suppose matrix-valued $f_n$ converge uniformly to a continuously differentiable $f$, their total variations are bounded, and vector-valued $g_n$ converge uniformly to a continuous $g$. Define $\int f_n\,dg_n$ by discrete summation and integration by parts. These integrals converge to
$f(T)g(T)-f(0)g(0)-\int_0^T df(t)g(t)$.
To see this, approximate $g$ uniformly by a piecewise linear function. The approximation error in integration by parts is bounded by the approximation's uniform error times the total variation of $f_n$ and its endpoint norms. For the piecewise linear approximation, uniform convergence of $f_n$ suffices. Left- versus right-endpoint conventions contribute at most
$\operatorname{TV}(f_n)\max_j\norm{\Delta g_n(j)}$, which vanishes when the limit of $g_n$ is continuous.
The same argument applies when total variations are only bounded in probability, by localization and the usual subsequence representation of convergence in distribution to continuous paths.

Applying this principle to Lemmas~\ref{lem:score} and \ref{lem:products} proves \eqref{eq:recent}. The limiting integral is the ordinary stochastic integral with deterministic integrand $e^{-tH}$. No stochastic predictability of the finite-$a$ backward weights is used.

\subsection{Uniform control of the remote past}
\begin{lemma}[Remote-past localization]\label{lem:remote}
Under Assumptions~\ref{ass:D}--\ref{ass:S}, let
$N=\lceil T/a\rceil$ and
$R_{a,T}=\sqrt a\sum_{j\ge N}P_j(a)\xi_{-j}$.
There are deterministic $C,c'>0$ such that for $0<a\le a_0\wedge1$, $T\ge1$, $M\ge1$, and $\epsilon>0$,
\[
 \Prob(\norm{R_{a,T}}>\epsilon)
 \le\Prob(K\vee D>M)+\epsilon^{-1}CM^2(1+T)^{3/2}e^{-c'T},
\]
where $D=\sup_{n\ge1}n^{-1}\sum_{j<n}\norm{H_{-j}}<\infty$ almost surely. In particular,
\begin{equation}\label{eq:remotelimit}
 \lim_{T\to\infty}\limsup_{a\downarrow0}
 \Prob(\norm{R_{a,T}}>\epsilon)=0.
\end{equation}
\end{lemma}
\begin{proof}
Since $\xi_0\in L^1$, ergodicity gives $S_n/n\to0$ almost surely. The boundary term at infinity therefore vanishes when summing $P_j(S_{j+1}-S_j)$ by parts. The identity $P_j-P_{j+1}=aP_jH_{-j}$ proves \eqref{eq:abeltail}. Its series converges pathwise: the products decay exponentially, whereas $H_{-j}$ and $S_j$ grow at most polynomially almost surely by Lemma~\ref{lem:weights} and the ergodic theorem.

On $E_M=\{K\le M,D\le M\}$, put $b_\ell=2^\ell N$. For $b_\ell\le j<2b_\ell$, sum the matrix norms before taking any expectation:
\[
 \sum_{j=b_\ell}^{2b_\ell-1}e^{-caj}\norm{H_{-j}}\norm{S_{j+1}}
 \le 2Db_\ell e^{-cab_\ell}\max_{k\le2b_\ell}\norm{S_k}.
\]
Consequently,
\begin{align*}
 \E\bigl[\mathbf1_{E_M}\norm{R_{a,T}}\bigr]
 &\le M\sqrt a e^{-caN}\E\norm{S_N}
   +2M^2a^{3/2}\sum_{\ell\ge0}b_\ell e^{-cab_\ell}
           \E\max_{k\le2b_\ell}\norm{S_k}\\
 &\le CM^2\left\{\sqrt{aN}e^{-caN}
             +\sum_{\ell\ge0}(ab_\ell)^{3/2}e^{-cab_\ell}\right\},
\end{align*}
by Lemma~\ref{lem:score}. Since $T\le aN\le T+1$, the last expression is at most
$CM^2(1+T)^{3/2}e^{-c'T}$ after reducing $c'>0$ if necessary. Markov's inequality on $E_M$, and the bound $\Prob(E_M^c)$ on its complement, prove the claimed inequality. Independence between $K$, $D$, and the score is never needed. For fixed $M$, let $T\to\infty$, and then let $M\to\infty$. This proves \eqref{eq:remotelimit}.
\end{proof}

\begin{proof}[Proof of the marginal assertion of Theorem~\ref{thm:clt}]
The series representation gives
\[
 a^{-1/2}e_0^{(a)}=
 \sqrt a\sum_{j<\lceil T/a\rceil}P_j(a)\xi_{-j}+R_{a,T}.
\]
The first term converges by \eqref{eq:recent} to
$G_T=\int_0^T e^{-tH}\Omega^{1/2}\,dW_t$.
Since $H\succ0$, $G_T$ converges in $L^2$ as $T\to\infty$ to a centered Gaussian vector with covariance $V$ in \eqref{eq:lyapunov}. Lemma~\ref{lem:remote} and the converging-together argument prove \eqref{eq:clt}. Differentiating $e^{-tH}\Omega e^{-tH}$ and integrating over $t\ge0$ gives $HV+VH=\Omega$. If $HV+VH=0$, diagonalization of $H$ shows each entry in its eigenbasis is multiplied by a strictly positive sum of eigenvalues, proving uniqueness. No expectation of the stationary iterate has been passed to a limit.
\end{proof}

\begin{proof}[Proof of Proposition~\ref{prop:bl}]
Let $f$ be bounded by one and have Lipschitz constant at most one. On $E_M=\{K\le M,D\le M\}$, the absolute difference $|f(U_{a,T}+R_{a,T})-f(U_{a,T})|$ is at most $\norm{R_{a,T}}$; on $E_M^c$ it is at most two. The localized expectation estimate in the proof of Lemma~\ref{lem:remote} bounds the first contribution by $CM^2(1+T)^{3/2}e^{-c'T}$, and the second by $2\Prob(E_M^c)$. Couple $G_T$ and $G_\infty$ using the same Brownian motion in their integral representations. Then $\E\norm{G_\infty-G_T}\le Ce^{-\lambda_{\min}(H)T}$ by the Gaussian covariance integral. The triangle inequality, followed by the supremum over $f$, proves \eqref{eq:blbound}.
\end{proof}

\subsection{Stationary process convergence and initialization}
\begin{proof}[Proof of the process assertion of Theorem~\ref{thm:clt}]
Fix $L,T>0$ and initialize an auxiliary recursion at zero at index $-\lceil L/a\rceil$. On the interval $[-L,T]$, the cumulative centered Gram matrices converge uniformly to zero after multiplication by $a$, and the cumulative sums of Gram norms are bounded after that scaling. The proof of Lemma~\ref{lem:products} applies uniformly to every subinterval: its bounds involve only the total Gram norm over the whole interval and twice the maximum cumulative centered Gram sum. Thus the two-parameter transition matrices converge uniformly to $e^{-H(t-s)}$, $-L\le s\le t\le T$. Their total variations in $s$ are uniformly bounded.

Lemma~\ref{lem:score} gives a joint score functional CLT on $[-L,T]$. The integration-by-parts continuity principle used above, now uniformly in the endpoint $t$, shows that the rescaled auxiliary recursion converges in $D([0,T],\R^p)$ to
\[
 U^{[L]}(t)=\int_{-L}^t e^{-H(t-s)}\Omega^{1/2}\,dW_s.
\]
To make the functional step explicit, take a Skorokhod representation of the joint score and transition-matrix convergence, localizing their variation bounds. Continuity of the score limit gives uniform score convergence on this representation. Approximate that limit by a single piecewise linear function on $[-L,T]$; the variation and endpoint bounds control the approximation error uniformly in $t$. The resulting deterministic integrals converge uniformly by the uniform transition-matrix limit. This proves weak convergence of the auxiliary processes in the $J_1$ topology, and hence their tightness, on the original probability space.

At time zero, the discrepancy between the stationary and auxiliary recursions is the remote-past term $R_{a,L}$, up to the harmless rounding of $L/a$. Over the future interval this discrepancy is multiplied by forward transitions. Their largest norm is bounded by
\[
 \exp\left\{a\sum_{i=1}^{\lceil T/a\rceil}\norm{H_i}\right\}=O_P(1),
\]
independently of $L$. Denote this bound by $M_{a,T}$, the rescaled stationary process by $U_a$, and the rescaled auxiliary process by $U_a^{[L]}$. For every $B\ge1$,
\[
\Prob\left(\sup_{t\le T}\norm{U_a(t)-U_a^{[L]}(t)}>\epsilon\right)
\le\Prob(M_{a,T}>B)+\Prob(\norm{R_{a,L}}>\epsilon/B).
\]
Lemma~\ref{lem:remote}, followed by $B\to\infty$, gives
$\lim_{L\to\infty}\limsup_{a\downarrow0}\Prob(\sup_{t\le T}\norm{U_a(t)-U_a^{[L]}(t)}>\epsilon)=0$.
On a common two-sided Brownian motion, $U^{[L]}$ converges uniformly in probability on $[0,T]$ as $L\to\infty$ to
$U(t)=\int_{-\infty}^t e^{-H(t-s)}\Omega^{1/2}\,dW_s$.
The difference equals $e^{-Ht}$ times a centered Gaussian initial discrepancy whose variance tends to zero exponentially in $L$. The converging-together argument in $D([0,T],\R^p)$ gives the stationary OU limit, whose initial value is independent of future Brownian increments.
\end{proof}

\begin{proof}[Proof of Corollary~\ref{cor:startup}]
Couple the initialized and stationary recursions on the same observations. By \eqref{eq:forward},
\[
 \frac{\norm{\beta_{n(a)}-\beta_{n(a)}^{(a)}}}{\sqrt a}
 \le K^+e^{-can(a)}
 \left\{\frac{\norm{\beta_0-\beta^*}}{\sqrt a}
          +\frac{\norm{e_0^{(a)}}}{\sqrt a}\right\}.
\]
The second term in braces is tight by Theorem~\ref{thm:clt}, and $K^+$ is finite almost surely and independent of $a$ as a random variable on the common environment. Equation~\eqref{eq:burnin} makes the right side converge to zero in probability. Stationarity and Slutsky's theorem finish the proof.
\end{proof}

\section{Exact conditional covariance and total variation}\label{app:gaussian}
\begin{proof}[Proof of Theorem~\ref{thm:gaussian}]
The assumptions imply the score condition, because $X_iZ_i$ is an $L^2$ martingale difference for the joint design-innovation and error filtration. Conditional on $\G_X$, the series \eqref{eq:perpetuity} is a linear series of independent centered Gaussian variables. Its covariance is
\[
 C_a=\sigma^2a^2\sum_{j\ge0}P_jH_{-j}P_j\trans.
\]
This covariance is finite almost surely: its norm is bounded by
$\sigma^2a^2K^2\sum_{j\ge0}e^{-2caj}\norm{H_{-j}}$, which is finite by Lemma~\ref{lem:weights}. Thus the conditional distribution is a well-defined Gaussian law, regardless of unconditional moments.

Symmetry of $H_{-j}$ gives the exact matrix identity
\[
 P_jP_j\trans-P_{j+1}P_{j+1}\trans
 =2aP_jH_{-j}P_j\trans-a^2P_jH_{-j}^2P_j\trans.
\]
The boundary $P_nP_n\trans$ vanishes by \eqref{eq:uniformpath}. Summing and dividing by $2$ proves
\begin{equation}\label{eq:gaussian-weights}
a\sum_{j\ge0}P_jH_{-j}P_j\trans=\tfrac12(I+\mathcal R_a),
\qquad C_a/a=\sigma^2(I+\mathcal R_a)/2.
\end{equation}
Let $h_j=\norm{H_{-j}}$. For all small $a$,
\[
 \norm{\mathcal R_a}\le K^2a^2\sum_{j\ge0}e^{-2caj}h_j^2\longrightarrow0\quad\as
\]
by \eqref{eq:weighted-square}. The common, step-uniform random bound in Theorem~\ref{thm:stability} is what permits this conclusion without a second moment of $H_0$.

Let $r_1,\ldots,r_p\ge0$ be the eigenvalues of $\mathcal R_a$. The Kullback--Leibler divergence from the conditional Gaussian law in \eqref{eq:conditional} to $N_p(0,\sigma^2I/2)$ is
\[
 \frac12\sum_{j=1}^p\{r_j-\log(1+r_j)\}
 \le\frac14\sum_{j=1}^pr_j^2.
\]
Pinsker's inequality, for $d_{\TV}(\mu,\nu)=\sup_A|\mu(A)-\nu(A)|$, gives \eqref{eq:tvbound}. The bound tends to zero almost surely and is at most one. Integrating the conditional bound and applying bounded convergence proves unconditional total-variation convergence.
\end{proof}

\begin{corollary}[Finite-step central coverage]\label{cor:coverage-bound}
Under Theorem~\ref{thm:gaussian}, for any $0<\alpha<1$ and every sufficiently small fixed $a$,
\[
 \Prob\left(\frac{2\norm{B_a-\beta^*}^2}{a\sigma^2}\le\chi^2_{p,1-\alpha}\right)
 \le1-\alpha.
\]
The left side converges to $1-\alpha$ as $a\downarrow0$.
\end{corollary}
\begin{proof}
Conditional on $\G_X$, diagonalize $\mathcal R_a$ and represent the statistic as
$\sum_{j=1}^p(1+r_j)G_j^2$ with iid standard normal $G_j$. It is at least $\sum_jG_j^2$ on this coupling. The bound follows conditionally and hence unconditionally. Total-variation convergence proves the limit. This finite-step inequality uses the population noise variance and stationarity; it is not asserted for a finite-history recursion or for an estimated variance.
\end{proof}

\section{AR(1) laws with no positive stationary moments}\label{sec:examples}
Consider $p=1$, independent standard Gaussian errors, and
\begin{equation}\label{eq:ar}
 X_i=\rho X_{i-1}+\eta_i=\sum_{k\ge0}\rho^k\eta_{i-k},\qquad 0<\rho<1,
\end{equation}
with iid positive innovations satisfying $\eta_i\ge1$. Then $X_i\ge(1-\rho)^{-1}$. Whenever $\eta_0\in L^q$, $q\ge1$,
\begin{equation}\label{eq:ardelta}
 \delta_{X,q}(k)=\rho^k\norm{\eta_0-\eta_0'}_q.
\end{equation}

\begin{theorem}[Persistent extremes destroy every stationary moment]\label{thm:nomoments}
For either innovation law below, \eqref{eq:ar} satisfies Assumptions~\ref{ass:D}--\ref{ass:S}. For every sufficiently small fixed $a>0$,
\begin{equation}\label{eq:allinfinite}
 \int|b-\beta^*|^r\,\pi_a(db)=\infty\quad\text{for every }r>0,
 \qquad
 d_{\TV}\!\left(\Law\left(\frac{B_a-\beta^*}{\sqrt a}\right),N(0,1/2)\right)\to0.
\end{equation}
\begin{enumerate}[label=(\roman*),leftmargin=*,nosep]
\item \emph{Pareto innovations:} $\Prob(\eta_0>u)=u^{-\kappa}$ for $u\ge1$, $\kappa>2$. For each $r>0$, $\E|P_m(a)|^r=\infty$ whenever $2mr\ge\kappa$.
\item \emph{All-moment innovations:} $\Prob(\eta_0>u)=\exp\{-(\log u)^\nu\}$ for $u\ge1$, $1<\nu<2$. Every finite product and every observation $(X_i,Y_i,Z_i)$ has all positive moments, and the observations have geometric functional dependence at every finite moment order. The stationary iterate nevertheless satisfies \eqref{eq:allinfinite}.
\end{enumerate}
For sufficiently large $x$, the stationary tail has a lower bound
\begin{equation}\label{eq:slowtail}
 \Prob(|B_a-\beta^*|>x)\ge c_a\exp\{-C_a(\log x)^\gamma\},
 \quad \gamma=\tfrac12\text{ in (i)},\quad \gamma=\tfrac\nu2\text{ in (ii)}.
\end{equation}
\end{theorem}

The mechanism can be seen from one innovation $U=\eta_{-m+1}$. It enters all $m$ design values in $P_m(a)$, and, for
$U\ge\sqrt{2/a}\,\rho^{-(m-1)}$,
\begin{equation}\label{eq:productlower}
 |P_m(a)|\ge(a/2)^m\rho^{m(m-1)}U^{2m}.
\end{equation}
Conditional on the design, the stationary error is Gaussian with variance at least
$a^2(1-\rho)^{-2}|P_m(a)|^2$. In the Pareto case, \eqref{eq:productlower} immediately makes every prescribed stationary moment infinite by choosing $m$ sufficiently large. For the second law, choose $U>e^{dm}$ with $d>\log(1/\rho)$. The log amplification is of order $m^2$, whereas the negative log probability is of order $m^\nu=o(m^2)$. Thus observation moments of every finite order coexist with infinite stationary iterate moments. The product and moment arguments are in Appendices~\ref{app:ar-product}--\ref{app:ar-logweibull}; Appendix~\ref{app:ar-tail} proves \eqref{eq:slowtail}.

The Pareto example tests the qualitative theory under a finite second design moment. The log-Weibull example additionally has all finite product moments and satisfies the first-order approximation assumptions, so it tests the distinction between finite-product integrability and stationary moment stability.

\subsection{Proofs for the persistent-innovation examples}\label{app:examples}
Let $\lambda_\rho=\log(1/\rho)>0$ and $\ell=(1-\rho)^{-1}$. Positive innovations bounded below by one give $X_i\ge\ell$. For $q\ge1$ with $\eta_0\in L^q$, Minkowski's inequality gives
$\norm{X_0}_q\le\norm{\eta_0}_q/(1-\rho)$, and replacing $\eta_0$ gives the exact dependence coefficient \eqref{eq:ardelta}. The independent Gaussian errors may be included in the iid innovations defining the causal process. The score has zero conditional mean given the past and finite second moment whenever $\E\eta_0^2<\infty$.

\subsection{A lower bound from a single persistent innovation}\label{app:ar-product}
Fix $m\ge1$ and put $U=\eta_{-m+1}$. For $0\le j<m$,
$X_{-j}\ge\rho^{m-1-j}U$. On
\[
 U\ge u_m:=\sqrt{2/a}\,\rho^{-(m-1)},
\]
each $aX_{-j}^2\ge2$, so
$|1-aX_{-j}^2|\ge(a/2)X_{-j}^2$.
Multiplying proves \eqref{eq:productlower}. Let $Q_a$ be the unscaled conditional variance of $e_0^{(a)}$. The $j=m$ summand in its series gives
\begin{equation}\label{eq:varlower}
 Q_a\ge a^2P_m(a)^2X_{-m}^2\ge a^2\ell^2P_m(a)^2.
\end{equation}
For $r>0$, Tonelli's theorem and the conditional Gaussian moment formula imply
\begin{equation}\label{eq:momentlower}
 \E|e_0^{(a)}|^r=c_r\E Q_a^{r/2}
 \ge c_r(a\ell)^r\E|P_m(a)|^r,
 \qquad c_r=\E|N(0,1)|^r>0.
\end{equation}
These identities allow infinite expectations.

\subsection{Pareto innovations}\label{app:ar-pareto}
For $\Prob(\eta_0>u)=u^{-\kappa}$, $u\ge1$, with $\kappa>2$, the design and score satisfy the main assumptions. For each $r>0$ and $2mr\ge\kappa$,
\[
 \E|P_m(a)|^r\ge(a/2)^{mr}\rho^{rm(m-1)}
       \E[U^{2mr}\mathbf1_{\{U\ge u_m\}}]=\infty.
\]
Equation~\eqref{eq:momentlower} proves the failure of every positive stationary moment. Notice the quantifiers: a finite product can retain sufficiently small fractional moments, but for each fixed positive order a sufficiently long product loses that moment.

\subsection{All moments of observations, but none of the stationary iterate}\label{app:ar-logweibull}
For $1<\nu<2$, let $\Prob(\eta_0>u)=\exp\{-(\log u)^\nu\}$ for $u\ge1$. One explicit representation is $\eta_0=\exp(E_0^{1/\nu})$ with $E_0$ standard exponential. For any $q>0$,
\[
 \E\eta_0^q=1+q\int_0^\infty\exp\{qt-t^\nu\}\,dt<\infty.
\]
Thus the AR(1) design, errors, and responses have all moments, as do all finite products by H\"older's inequality. Equation~\eqref{eq:ardelta} gives geometric functional dependence at every finite order $q\ge1$. For $k\ge1$, the score coefficient also obeys
\[
 \delta_{\xi,q}(k)=\rho^k\norm{\eta_0-\eta_0'}_q\norm{Z_0}_q,
\]
and its lag-zero coefficient is finite. The response coefficients are bounded by the corresponding design and error coefficients.

Fix $r>0$ and select $d>\lambda_\rho$. For all sufficiently large $m$, the event $\{U>e^{dm}\}$ is contained in $\{U\ge u_m\}$. The product lower bound yields
\begin{align*}
 \log\E|P_m(a)|^r
 &\ge rm\log(a/2)-r\lambda_\rho m(m-1)+2rdm^2-d^\nu m^\nu\\
 &=r(2d-\lambda_\rho)m^2-d^\nu m^\nu+O_{a,r,\rho}(m)\longrightarrow\infty.
\end{align*}
For each fixed $m$ this expectation is finite; its divergence as $m\to\infty$, together with \eqref{eq:momentlower}, forces $\E|e_0^{(a)}|^r=\infty$. Since $r$ was arbitrary, all positive stationary moments fail. Theorem~\ref{thm:gaussian} applies to both innovation laws and proves the total-variation limit in Theorem~\ref{thm:nomoments}.

\subsection{Subpolynomial stationary lower tails}\label{app:ar-tail}
On $\{U\ge u_m\}$, equation~\eqref{eq:productlower} simplifies to
\[
 |P_m(a)|\ge e^{\lambda_\rho m(m-1)}.
\]
For large $x$, choose the smallest integer $m\ge2$ satisfying
$a\ell e^{\lambda_\rho m(m-1)}\ge x$. Then
$m=O_{a,\rho}(\sqrt{\log x})$ and, by \eqref{eq:varlower}, $Q_a^{1/2}\ge x$ on that event. Conditional Gaussianity gives
\[
 \Prob(|e_0^{(a)}|>x)
 \ge \Prob(|N(0,1)|>1)\Prob(\eta_0\ge u_m).
\]
For Pareto innovations, the last probability equals
$(a/2)^{\kappa/2}\exp\{-\kappa\lambda_\rho(m-1)\}$.
For the all-moment innovations it equals
$\exp\{-[\tfrac12\log(2/a)+\lambda_\rho(m-1)]^\nu\}$.
These expressions give \eqref{eq:slowtail} with constants depending on $a$ and the innovation law. Since $\gamma<1$, $x^r\Prob(|e_0^{(a)}|>x)\to\infty$ for every $r>0$, which also directly proves the moment assertion. No matching upper tail estimate is claimed.

\subsection{Why persistence is essential to the all-moment failure}
For comparison, suppose $p=1$ and the observations are iid over time, with $\E X_0^2\in(0,\infty)$ and $\E|\xi_0|^r<\infty$ for some $0<r\le1$. The function
$t\mapsto\bigl||1-t|^r-1\bigr|$ is bounded by $C_rt$ on $[0,\infty)$. Dominated convergence therefore gives
\[
 \E|1-aX_0^2|^r=1-ra\E X_0^2+o(a)<1
\]
for sufficiently small $a$. Subadditivity for order $r\le1$, and independence of the score from subsequent update factors, imply
\[
 \E|e_0^{(a)}|^r
 \le a^r\E|\xi_0|^r\sum_{j\ge0}\{\E|1-aX_0^2|^r\}^j<\infty.
\]
Thus iid designs with these moments retain some positive stationary moments at small steps. In the examples, temporal persistence of the same large innovation is responsible for losing all of them.

\section{Inference, preconditioning, and fixed mini-batches}\label{app:extensions}
\subsection{Feasible central inference with independent errors}
Suppose $(Z_i)$ is iid, centered, with variance $\sigma^2\in(0,\infty)$, independent of the design process. Gaussianity is not required here. The score is a martingale difference, and
$\Omega=\E X_0X_0\trans Z_0^2=\sigma^2H$, giving $V=\sigma^2I/2$.
Define
\[
 G_m=\frac1m\sum_{i=1}^mX_iX_i\trans,\qquad
 \widetilde\beta_m=G_m^{-1}\frac1m\sum_{i=1}^mX_iY_i,
 \qquad
 \widehat\sigma_m^2=\frac1m\sum_{i=1}^m(Y_i-X_i\trans\widetilde\beta_m)^2.
\]
On the event that $G_m$ is singular, assign any finite value to $\widetilde\beta_m$; this event eventually does not occur almost surely because $G_m\to H\succ0$.
The ergodic theorem gives
$m^{-1}\sum_{i=1}^m\xi_i\to0$ almost surely, so
$\Delta_m:=\widetilde\beta_m-\beta^*=G_m^{-1}m^{-1}\sum_i\xi_i\to0$ almost surely. Expanding the residual variance,
\[
 \widehat\sigma_m^2
 =\frac1m\sum_{i=1}^mZ_i^2
   -2\Delta_m\trans\frac1m\sum_{i=1}^m\xi_i
   +\Delta_m\trans G_m\Delta_m\longrightarrow\sigma^2\quad\as.
\]
Theorem~\ref{thm:clt} or Corollary~\ref{cor:startup}, combined with Slutsky's theorem, gives the chi-square limit in \eqref{eq:confidence}. The calibration size can be any deterministic $m(a)\to\infty$; choosing $m(a)=n(a)$ uses only observations already processed. Independence between the variance estimator and the final SGD iterate is unnecessary. The computation can be implemented using the cumulative Gram matrix, cross-product vector, and sum of $Y_i^2$; it does not estimate the variance from stationary SGD iterates.

\subsection{Fixed positive definite preconditioning}
Let $M\succ0$ be deterministic and consider
\[
 \beta_i=\beta_{i-1}+aMX_i(Y_i-X_i\trans\beta_{i-1}).
\]
With $u_i=M^{-1/2}(\beta_i-\beta^*)$, the transformed recursion has positive semidefinite coefficient
$\widetilde H_i=M^{1/2}H_iM^{1/2}$ and score
$\widetilde\xi_i=M^{1/2}\xi_i$. All assumptions of the main theorems are preserved. Transforming back, the stationary limiting covariance $V_M$ solves
\[
 MH V_M+V_M HM=M\Omega M.
\]
Under independent homoscedastic errors, $V_M=\sigma^2M/2$. Under independent Gaussian errors, the conditional total-variation statement also transfers through the invertible linear transformation. Thus the central region in \eqref{eq:confidence} uses the quadratic form $\norm{M^{-1/2}(\beta-b)}^2$ for this algorithm.

\subsection{Fixed nonoverlapping mini-batches}
Fix an integer $b\ge1$. Updating once per block with the average of the $b$ sample gradients gives
\[
 e_k=(I-a\overline H_k)e_{k-1}+a\overline\xi_k,\qquad
 \overline H_k=\frac1b\sum_{i=(k-1)b+1}^{kb}H_i,
 \quad
 \overline\xi_k=\frac1b\sum_{i=(k-1)b+1}^{kb}\xi_i.
\]
The coefficient is positive semidefinite, integrable, and has expectation $H$. Grouping iid innovations into blocks produces a causal representation. The proofs of Theorems~\ref{thm:stability}--\ref{thm:clt} apply to arbitrary positive semidefinite coefficients, not only rank-one matrices.
For the score, the block projection is the sum of the $b$ individual innovation projections in that block. Applying Minkowski's inequality and summing over future blocks bounds its projective coefficient sum by $\Theta_2$. Hence the score condition is inherited. Also,
\[
 \sum_{k=1}^n\overline\xi_k=\frac1b\sum_{i=1}^{nb}\xi_i,
 \qquad \overline\Omega=\Omega/b.
\]
The covariance of the limiting Gaussian law, on the block-update learning-time scale, is $V/b$.
For independent Gaussian errors, conditional block scores are independent centered Gaussian vectors with covariance $\sigma^2\overline H_k/b$. The telescoping calculation gives conditional covariance
$\sigma^2(I+\overline{\mathcal R}_a)/(2b)$ for the normalized stationary iterate, with $\overline{\mathcal R}_a\to0$ almost surely. These assertions concern fixed $b$; no joint growth regime in $b$, $p$, and $a^{-1}$ is asserted.

\section{Proofs for the Gaussian-refresh example}\label{app:refresh}

\begin{lemma}[One-coefficient anti-concentration]\label{up:lem:antic}
For any nondegenerate real random variable $Z$, there are $d,q_0>0$ such that
\[
 \inf_{v\in\R}\Prob(|v+cZ|>d|c|)\ge q_0\qquad(c\ne0).
\]
Consequently, if $e=\sum_j c_j Z_j$ is a convergent series with iid copies of $Z$ and the coefficients are fixed after conditioning on an independent environment, then, for any one index $m$ and $r>0$,
\begin{align}
 \E(|e|^r\mid\text{environment})&\ge q_0d^r|c_m|^r,\label{up:eq:anti-moment}\\
 \Prob(|e|>x)&\ge q_0\Prob(d|c_m|\ge x),\label{up:eq:anti-tail}
\end{align}
with an immaterial adjustment of strict inequalities.
\end{lemma}
\begin{proof}
Choose $z_1<z_2$ for which both $\Prob(Z\le z_1)$ and $\Prob(Z\ge z_2)$ are positive. Take $d=(z_2-z_1)/4$ and let $q_0$ be the smaller probability. An interval of length $2d$ cannot meet both separated sets. Hence at least one of these sets is disjoint from $\{z:|v+cz|\le d|c|\}$, proving the first claim. For the series, condition on the environment and all errors except $Z_m$; the remainder is a fixed shift. Integrating the conditional bound proves both assertions, also for infinite expectations. Almost sure convergence of the series makes this conditioning legitimate on a set of environments of full measure.
\end{proof}

\begin{proof}[Proof of Theorem~\ref{up:thm:refresh}]
Almost surely there is a most recent refresh at or before any fixed time; its location is a measurable function of the iid past. Setting $X_i$ equal to its Gaussian mark constructs the stationary causal solution. In $n$ future steps, no refresh has probability $\rho^n$; otherwise the latest Gaussian mark is independent of the starting state. This proves the kernel formula and, since the Gaussian law is nonatomic, the total-variation equality.

Replacing the innovation at time zero can affect $X_k$ only if there is no refresh at times $1,\ldots,k$. This event has probability $\rho^k$ and is independent of the time-zero coupled difference. Both coupled marginals are standard Gaussian. The $L^q$ triangle inequality gives the displayed dependence bounds. In particular, $H=1$ and all design moments exist. The score is a square-integrable martingale difference for the joint design/error innovation filtration, so all base theorems apply.

For every fixed $0<a<2/3$, Jensen's inequality gives
\[
\E\log|1-aG_0^2|\le\tfrac12\log\E(1-aG_0^2)^2
=\tfrac12\log(1-2a+3a^2)<0.
\]
The logarithm is integrable (its finitely many real zeros have integrable logarithmic singularities under a Gaussian density). Ergodicity makes the scalar product decay exponentially almost surely, and the score has a finite logarithmic moment. The stationary-series construction and coupling argument therefore apply throughout this interval, for any fixed $\rho<1$. This is a sufficient interval, not the maximal stability threshold.

Fix such an $a$ and any $r>0$. Isolate the coefficient of $Z_{-m}$:
\[
 c_{a,m}=aP_m(a)X_{-m}.
\]
On the event
\[
 A_{m,u}=\{|X_{-m}|\ge u,\ B_{-m+1}=\cdots=B_0=1\},
\]
all of the $m$ subsequent designs equal $X_{-m}$. For $au^2>2$,
\[
 |c_{a,m}|\ge au(au^2-1)^m,\qquad
 \Prob(A_{m,u})=\Prob(|G_0|\ge u)\rho^m.
\]
Lemma~\ref{up:lem:antic} therefore implies
\[
 \E|e_0^{(a)}|^r\ge
 q_0d^r(au)^r\Prob(|G_0|\ge u)
 \{\rho(au^2-1)^r\}^{m}.
\]
Choose a finite $u$ large enough that the quantity in braces exceeds one, then let $m\to\infty$. The left side is independent of $m$, so it is infinite. This proof works for each $r>0$ without assuming any stationary moments. The Gaussian-limit claims are the preceding theorems.

When $\rho=0$, the current design/error pair is independent of the previous iterate. Now
\[
 \E(1-aX_i^2)^2=1-2a+3a^2<1\quad(0<a<2/3).
\]
The independent affine recursion has a square-integrable stationary solution. Its second-moment recursion is $M=(1-2a+3a^2)M+a^2\sigma^2$, which gives~\eqref{up:eq:iid-var}.
\end{proof}

\begin{remark}[Two particularly transparent error choices]
With independent Gaussian errors, each time point $(X_i,Z_i,Y_i)$ is jointly Gaussian (possibly singular as a triple), although the full time series is \emph{not} a jointly Gaussian process. With independent Rademacher errors, every observation is sub-Gaussian, the errors are bounded, and $X_iZ_i$ has an exactly Gaussian marginal distribution. The all-moment failure persists in both cases. Only the Gaussian-error case is assigned total-variation convergence here; bounded errors by themselves do not justify that metric.
\end{remark}

\begin{corollary}[A subpolynomial lower tail in the Gaussian-refresh example]
For the model in Theorem~\ref{up:thm:refresh}, there are positive constants depending on $a,\rho$ and the error distribution such that, for all sufficiently large $x$,
\[
 \Prob(|e_0^{(a)}|>x)\ge c\exp\left\{-C\frac{\log x}{\log\log x}\right\}.
\]
This is a lower bound, not a matching asymptotic.
\end{corollary}
\begin{proof}
Write $y=\log x$ and take $u=\sqrt y/\log y$. For large $y$, $b=au^2-1>1$ and $\log b\ge\tfrac12\log y$. Choose
\[
 m=\left\lceil\frac{\log(x/(adu))}{\log b}\right\rceil.
\]
For all sufficiently large $y$, $m\le3y/\log y$ and $dau b^m\ge x$. The event in the preceding proof and~\eqref{up:eq:anti-tail} give $\Prob(|e|>x)\ge q_0\rho^m\Prob(|G|\ge u)$. The elementary Gaussian lower tail bound $\Prob(|G|\ge u)\ge c\,u(1+u^2)^{-1}e^{-u^2/2}$ shows that its negative logarithm is at most $C y/\log y$. This proves the claim.
\end{proof}

\section{First-order conditional covariance under a fourth design moment}

\begin{proof}[Proof of Theorem~\ref{up:thm:first}]
Let $h_j=\norm{H_{-j}}$. On a probability-one event, the preceding finite-window lemma yields uniform convergence of $P_j(a)$ to $e^{-ajH}$ for $j\le T/a$, and $a\sum_{j<T/a}h_j^2=O(1)$ for fixed $T$. Hence
\[
 a\sum_{j<T/a}\{P_jH_{-j}^2P_j^\top-e^{-ajH}H_{-j}^2e^{-ajH}\}\longrightarrow0.
\]
For each matrix entry, the ergodic theorem and summation by parts against the continuously differentiable deterministic matrix weights give
\[
 a\sum_{j<T/a}e^{-ajH}H_{-j}^2e^{-ajH}
 \longrightarrow\int_0^T e^{-tH}\E(H_0^2)e^{-tH}\,dt.
\]
To justify the weighted ergodic step explicitly, centered partial sums of $H_{-j}^2$ are $o(n)$; multiplying by $a$ makes their maximum up to $T/a$ vanish. Abel summation then bounds the weighted error by this maximum times the finite total variation and endpoint norms of the deterministic weights.

Let $D_2=\sup_{n\ge1}n^{-1}\sum_{j<n}h_j^2<\infty$ almost surely. The tail after $N=\lceil T/a\rceil$ is bounded using~\eqref{eq:uniformpath} and dyadic intervals:
\[
 a\sum_{j\ge N}\norm{P_jH_{-j}^2P_j^\top}
 \le CK^2D_2(1+T)e^{-c'T}.
\]
Indeed, on $2^\ell N\le j<2^{\ell+1}N$, sum $h_j^2$ first to get a bound $2^{\ell+1}ND_2$, and use the exponential weight at the left endpoint. This is uniform over $0<a\le a_0\wedge1$. Letting $T\to\infty$ proves~\eqref{up:eq:R-first}. Differentiation under the deterministic integral gives the Lyapunov identity.

The conditional Gaussian covariance formula and local Lipschitz continuity in total variation of nonsingular Gaussian laws with respect to their covariance matrices prove the final statements. The Lipschitz assertion follows directly from the Gaussian relative-entropy formula and Pinsker's inequality in a neighborhood of the identity.
\end{proof}

\section{Integrable logarithmic overshoots and an unconditional sharp rate}\label{app:rate}
Define $\delta_{H,q}(k)=\norm{H_k-H_k^{\{0\}}}_q$, using operator norm inside $L^q$. We use condition~\eqref{up:eq:rate-ass}.
A directly checkable sufficient condition is $X_0\in L^{2q}$ and $\sum_k\delta_{X,2q}(k)<\infty$, because
\[
 \delta_{H,q}(k)\le2\norm{X_0}_{2q}\delta_{X,2q}(k).
\]
The qualitative results retain their weaker conditions.

\begin{lemma}[A step-sensitive prefactor with integrable logarithm]\label{up:lem:overshoot}
Under~\eqref{up:eq:rate-ass} and Assumption~\ref{ass:D}, there exist deterministic $a_0,c>0$ and $L_*\ge0$ with $\E L_*<\infty$ such that
\begin{equation}\label{up:eq:sensitive}
 \norm{P_j(a)}\le e^{aL_*}e^{-caj}
 \quad\text{for every }j\ge0,\ 0<a\le a_0.
\end{equation}
Only integrability of $L_*$ is needed in the probability-law bounds below.
\end{lemma}
\begin{proof}
Follow the preceding block proof, choosing a fixed block length $m$ and a small $a_0\le1$. For backward block $r$, let $s_r$ be the sum of the Gram norms and $g_r$ the smallest eigenvalue of the sum of the Grams. Put
\[
 z_r=-g_r+2s_r\min(a_0s_r,1),\qquad \E z_r=-\eta<0.
\]
The existing deterministic product inequality bounds each block log norm by $az_r$ simultaneously in $a\le a_0$. With $W_n=\sum_{r=1}^n(z_r+\eta)$ and $W_0=0$, define
\[
 L=\sup_{n\ge0}\{W_n-\eta n/2+s_{n+1}\}.
\]
The block bound gives
\[
 \log\norm{P_j(a)}\le a(L+\eta/2)-\frac{a\eta}{2m}j.
\]
It remains to prove $\E L<\infty$, rather than freezing $e^{aL}$ at $a=a_0$.

First, $s_r,z_r\in L^q$. The function $s\mapsto 2s\min(a_0s,1)$ is globally Lipschitz with constant at most four. The minimum eigenvalue is Lipschitz in operator norm. Thus the block functional defining $z_r$ is Lipschitz, with constant at most five relative to the sum of the changes in its $m$ Gram matrices. Introduce grouped iid innovations and their projections,
\[
E_k=(\varepsilon_{(k-1)m+1},\ldots,\varepsilon_{km}),\quad
\F_k^{(m)}=\sigma(E_i:i\le k),\quad
\Proj_k^{(m)}U=\E(U\mid\F_k^{(m)})-\E(U\mid\F_{k-1}^{(m)}).
\]
Let $s_k^+=\sum_{i=(k-1)m+1}^{km}\norm{H_i}$ and define
$z_k^+=-\lambda_{\min}(\sum_{i=(k-1)m+1}^{km}H_i)+2s_k^+\min(a_0s_k^+,1)$.
This forward-indexed functional is causal for $(E_k)$, and the backward block satisfies $z_r=z_{1-r}^+$. Define $\delta_{z^+,q}^{\mathrm{block}}(k)$ by replacing $E_0$ with an independent copy. Replacing its $m$ innovations one at a time gives
\[
 \sum_{k\ge0}\delta_{z^+,q}^{\mathrm{block}}(k)
 \le 5m\sum_{u\ge0}\delta_{H,q}(u)<\infty.
\]
Here a coefficient at a negative lag is interpreted as zero. For a fixed replaced time $\ell\in\{1-m,\ldots,0\}$, the intervals $\{(k-1)m+1,\ldots,km\}$, $k\ge0$, partition $\{1-m,2-m,\ldots\}$. Thus each lag $u=i-\ell\ge0$ occurs exactly once in the sum over future blocks and their time points, giving the factor $m$.

For the maximal bound, let $\theta_k=\norm{\Proj_0^{(m)}z_k^+}_q$. Conditional Jensen's inequality for the coupled block gives $\theta_k\le\delta_{z^+,q}^{\mathrm{block}}(k)$. The trivial remote-past tail of the iid block innovations and reverse martingale convergence yield the $L^q$ expansion
\[
z_i^+-\E z_0^+=\sum_{k\ge0}\Proj_{i-k}^{(m)}z_i^+.
\]
For fixed $k$, the summands indexed by $i$ form a martingale-difference sequence with $L^q$ norm $\theta_k$. The martingale inequality $\norm{\sum_i D_i}_q\le C_q(\sum_i\norm{D_i}_q^2)^{1/2}$ and Doob's inequality, followed by Minkowski's inequality over $k$, give
\[
\left\|\max_{j\le n}\left|\sum_{i=1}^j(z_i^+-\E z_0^+)\right|\right\|_q
\le C_q\sqrt n\sum_{k\ge0}\theta_k\le C\sqrt n.
\]
The summable projection norms justify passage from finite projection sums to the full expansion. Reversing a stationary window bounds the backward maximum by twice the corresponding forward maximum. Therefore
\begin{equation}\label{up:eq:Wmax}
 \norm{\max_{j\le n}|W_j|}_q\le C\sqrt n.
\end{equation}

Take $d=\eta/4$. We have
\[
 L\le \sup_{n\ge0}(W_n-dn)_+
       +\sup_{n\ge0}(s_{n+1}-dn)_+.
\]
For $N=2^k$,~\eqref{up:eq:Wmax} and the tail integral inequality yield
\[
 \E\big(\max_{j\le2N}|W_j|-dN\big)_+
 \le \frac{\E\max_{j\le2N}|W_j|^q}{(q-1)(dN)^{q-1}}
 \le C N^{1-q/2}.
\]
Summing over $k\ge0$ proves integrability of the first supremum since $q>2$. For the second,
\[
 \E\sup_{n\ge0}(s_{n+1}-dn)_+
 \le\sum_{n\ge0}\E(s_1-dn)_+\le C_d(\E s_1+\E s_1^2)<\infty.
\]
Thus $\E L<\infty$. Set $L_*=L+\eta/2$ and $c=\eta/(2m)$. Since the $n=0$ term in $L$ is $s_1\ge0$, $L_*\ge0$.
\end{proof}

The condition $q>2$ makes $\sum_k2^{k(1-q/2)}$ finite and is sufficient for this unconditional integrability step. The fourth design moment in Theorem~\ref{up:thm:first} instead controls the pathwise covariance expansion; at $q=2$, the displayed maximal bound alone does not sum over dyadic blocks.

\begin{proof}[Proof of Theorem~\ref{up:thm:rate}]
Let $E_a=\{aL_*\le1\}$ and $h_j=\norm{H_{-j}}$. By Lemma~\ref{up:lem:overshoot}, on $E_a$,
\[
 \norm{\mathcal R_a}\le e^2a^2\sum_{j\ge0}e^{-2caj}h_j^2.
\]
The exact conditional Gaussian law, Pinsker's inequality, and convexity of total variation under mixtures give
\begin{align*}
 d_{\TV}(\mu_a,\gamma_0)
 &\le \Prob(L_*>1/a)+C_pa^2\sum_{j\ge0}e^{-2caj}\E h_0^2\\
 &\le a\E L_*+C_p\frac{a^2\E h_0^2}{1-e^{-2ca}}\le Ca.
\end{align*}
This step integrates neither $e^{aL_*}$ nor the stationary covariance.

For the refined assertion, define $U_a=a\sum_{j\ge0}e^{-2caj}h_j^2$. This family is uniformly integrable. One direct proof chooses an increasing convex $\Psi$ with $\Psi(t)/t\to\infty$ and $\E\Psi(h_0^2)<\infty$. The sums of the deterministic weights are uniformly bounded. Normalize them and apply Jensen's inequality to bound $\E\Psi(U_a/C)$ uniformly. Such a $\Psi$ exists for any integrable nonnegative variable, by a piecewise-linear construction from its integrable tails.

Write $\gamma(R)=N_p(0,\sigma^2(I+R)/2)$. By Theorem~\ref{up:thm:first},
\[
 a^{-1}d_{\TV}(\gamma(\mathcal R_a),\gamma(aJ))\longrightarrow0
 \quad\text{almost surely}.
\]
On $E_a$, the triangle inequality via $\gamma(0)$ and the Gaussian Pinsker bound dominate this quantity by
\[
 C_p\big(\norm{\mathcal R_a}/a+\norm{J}\big)\ind_{E_a}
 \le C_p(e^2U_a+\norm{J}),
\]
which is uniformly integrable. Its expectation on $E_a$ therefore tends to zero. On $E_a^c$ use total variation at most one and
\[
 a^{-1}\Prob(L_*>1/a)\longrightarrow0,
\]
which follows from $\E L_*<\infty$. Mixture convexity proves~\eqref{up:eq:TVcorrected}.

It remains to evaluate a deterministic Gaussian comparison. After the common rescaling by $\sqrt2/\sigma$, differentiating the density of $N_p(0,I+aJ)$ at zero gives
\[
 \left.\frac{d}{da}f_a(x)\right|_{a=0}
   =\tfrac12(x^\top Jx-\tr J)\phi_p(x).
\]
This differentiation holds in $L^1$: for $a$ in a small interval, the derivatives are dominated by a polynomial times a Gaussian density with a fixed enlarged covariance. Therefore $d_{\TV}(\gamma_a,\gamma_0)/a$ tends to the right side of~\eqref{up:eq:TVconstant}. The reverse triangle inequality and~\eqref{up:eq:TVcorrected} transfer this limit to $\mu_a$. Finally, $\E H_0^2\succ0$: if $v^\top\E H_0^2v=0$ then $H_0v=0$ almost surely and hence $Hv=0$, contradicting $H\succ0$. Thus $J\succ0$ and the random quadratic form is not constant.
\end{proof}

\begin{corollary}[Quantitative coexistence in both all-moment examples]\label{up:cor:coexist}
The Gaussian-refresh construction of Theorem~\ref{up:thm:refresh}, with Gaussian errors, satisfies~\eqref{up:eq:TVrate}--\eqref{up:eq:TVconstant} and has no positive stationary moments. In this case $H=1$, $\E H_0^2=3$, and $J=3/2$, independently of $\rho$. In particular,
\[
 d_{\TV}\!\left(\Law(e_0^{(a)}/\sqrt a),N(0,\sigma^2(1+3a/2)/2)\right)=o(a),
\]
while the uncorrected distance is asymptotic to $\tfrac32\phi(1)a$.
The positive AR(1) construction with
$\Prob(\eta>u)=\exp\{-(\log u)^\nu\}$, $1<\nu<2$, also satisfies the assumptions of Theorem~\ref{up:thm:rate} for any finite $q>2$. Its qualitative no-moment conclusion therefore coexists with the same order-$a$ unconditional total-variation rate.
\end{corollary}
\begin{proof}
For the refresh model use its displayed $L^q$ dependence bound and Gaussian moments. For the positive AR(1) model, all innovation/design moments are finite and $\delta_{X,2q}(k)=\rho^k\norm{\eta_0-\eta_0'}_{2q}$; apply the sufficient condition preceding Lemma~\ref{up:lem:overshoot}. The scalar constant follows from $\E|G^2-1|=4\phi(1)$.
\end{proof}

\begin{remark}[Localization of the probability-law correction]
Even in these examples, $\E \mathcal R_a$ can be infinite for every fixed $a>0$. Localization on $\{aL_*\le1\}$ and a bounded probability distance make the expansion possible. The deterministic matrix $J$ corrects the Gaussian approximation to the law, while the conditional covariance retains the rare-environment fluctuations responsible for infinite moments.
\end{remark}

\subsection{First-order coverage from the probability-law expansion}\label{app:coverage-first}
\begin{proof}[Proof of Corollary~\ref{cor:coverage-first}]
Fix the scalar design law and $\alpha\in(0,1)$. Theorem~\ref{up:thm:rate} gives $d_{\TV}(\mu_a,\gamma_a)=o(a)$, where $\gamma_a=N(0,\sigma^2(1+aJ)/2)$. For $A=[-\sigma z/\sqrt2,\sigma z/\sqrt2]$,
\[
\mu_a(A)=\gamma_a(A)+o(a)
=2\Phi\!\left(\frac{z}{\sqrt{1+aJ}}\right)-1+o(a)
=1-\alpha-aJz\phi(z)+o(a).
\]
The last step is a Taylor expansion at $a=0$, with $z=z_{1-\alpha/2}$ and $J$ fixed. It proves~\eqref{eq:coverage-first}. For $A_a=\sqrt{1+aJ}\,A$, the Gaussian probability is exactly $\gamma_a(A_a)=1-\alpha$. The same TV bound applies uniformly to measurable sets, including $A_a$, proving~\eqref{eq:coverage-corrected}. Both arguments integrate bounded indicators and require no stationary iterate moment.

Under Corollary~\ref{cor:startup-tv}, $d_{\TV}(\mu_{a,n},\gamma_a)=o(a)$ under its stated history conditions. Replacing $\mu_a$ by $\mu_{a,n}$ therefore proves the finite-history versions. Mere consistency of an estimated noise variance gives an $o(1)$ coverage conclusion; retaining the first-order expansion requires a separate $o(a)$ control of its effect on coverage.
\end{proof}

\subsection{Transfer of the first-order expansion to finite histories}\label{app:finite-tv}
\begin{proof}[Proof of Corollary~\ref{cor:startup-tv}]
Put $A_i=I-aH_i$ and $P_{n,j}=A_n\cdots A_{n-j+1}$, with $P_{n,0}=I$ and $T_n=P_{n,n}$. Shifting Lemma~\ref{up:lem:overshoot} to endpoint $n$ gives
\[
\norm{P_{n,j}}\le e^{aL_{*,n}}e^{-caj},\qquad
L_{*,n}\stackrel{d}=L_*,\qquad \E L_*<\infty.
\]
The constants are independent of $n$; intersecting the shifted probability-one events makes the bound simultaneous in integer endpoints and $0<a\le a_0$.

Conditional on $\G_X$, the initialized normalized error is Gaussian with mean $m_n=a^{-1/2}T_nb$ and covariance
\[
S_n=\sigma^2a\sum_{j=0}^{n-1}P_{n,j}H_{n-j}P_{n,j}\trans.
\]
The stationary conditional covariance $S_\infty$ at the same endpoint extends this sum over all $j\ge0$. Both are finite almost surely. Telescoping the identity
$A_iA_i\trans=I-2aH_i+a^2H_i^2$ gives
\begin{equation}\label{eq:finite-covariance}
S_n=\frac{\sigma^2}{2}(I-T_nT_n\trans+\mathcal R_{a,n}),\qquad
\mathcal R_{a,n}=a^2\sum_{j=0}^{n-1}P_{n,j}H_{n-j}^2P_{n,j}\trans\succeq0.
\end{equation}
On $E_{a,n}=\{aL_{*,n}\le1\}$, choose $t_0$ so that $e e^{-ct_0}\le1/\sqrt2$. For $an\ge t_0$,
\[
S_\infty\succeq S_n\succeq\sigma^2I/4,\qquad
\norm{m_n}\le e\norm b\,a^{-1/2}e^{-can},
\]
and
\[
\ind_{E_{a,n}}\norm{S_\infty-S_n}
\le\sigma^2e^2a\sum_{j\ge n}e^{-2caj}\norm{H_{n-j}}.
\]
Taking expectations of the last bound uses only stationarity and $\E\norm{H_0}<\infty$, and gives $Ce^{-2can}$. The event indicator has already been bounded pathwise; no independence of $L_{*,n}$ and the design is used.

For covariance matrices $S_1,S_2\succeq\nu I$, Gaussian comparison gives
\[
d_{\TV}\{N_p(m,S_1),N_p(0,S_2)\}
\le C_{p,\nu}(\norm m+\norm{S_1-S_2}).
\]
Indeed, Pinsker's inequality controls the mean shift. Along the covariance segment $S_t=S_1+t(S_2-S_1)$, density differentiation and integration give an $L^1$ derivative bounded by $p\norm{S_2-S_1}/\nu$, since
$\E|G\trans D G-\tr D|\le2p\norm D$ for $G\sim N_p(0,I)$ and symmetric $D$. The covariance lower bound is the only spectral bound needed.

Apply this comparison on $E_{a,n}$, and use TV at most one on its complement. Convexity under mixtures gives \eqref{eq:startup-tv}. Finally,
\[
\frac{\Prob(L_*>1/a)}a
\le\E[L_*\ind_{\{L_*>1/a\}}]\longrightarrow0.
\]
The stated burn-in conditions make every term in \eqref{eq:startup-tv} $o(a)$. The triangle inequality transfers the corrected $o(a)$ result, and the reverse triangle inequality transfers the uncorrected first-order constant. The asserted history schedule has $an(a)\ge2\log^2(1/a)$ and therefore meets both sufficient conditions for every fixed $c>0$.
\end{proof}

\section{Quenched weak limits and non-Gaussian extensions}

\begin{proof}[Proof of Theorem~\ref{up:thm:quenched}]
For a fixed design path put $h_j=\norm{H_{-j}}$ and $w_{a,j}=\sqrt a\,P_j(a)X_{-j}$. Identity~\eqref{eq:gaussian-weights} gives
\[
 \sum_j w_{a,j}w_{a,j}^\top\to I/2.
\]
We also have $\max_j\norm{w_{a,j}}\to0$. Indeed, by stationarity and integrability, Borel--Cantelli gives $h_j/(j+1)\to0$ almost surely. For any $\varepsilon>0$, split the bound
\[
 \sup_j\norm{w_{a,j}}^2\le K^2 a\sup_j e^{-2caj}h_j
\]
at a finite random $J_\varepsilon$ after which $h_j\le\varepsilon(j+1)$. The initial part tends to zero; the remaining part has limsup at most $C K^2\varepsilon$. Let $\varepsilon\to0$.

Fix $v\in\R^p$. Conditional on the design, the summands $v^\top w_{a,j}Z_{-j}$ are independent and centered. Let $b_a=\sup_j|v^\top w_{a,j}|\to0$. For any $\varepsilon>0$, their Lindeberg sum is at most
\[
 \left(\sum_j|v^\top w_{a,j}|^2\right)
 \E\{Z_0^2\ind_{\{|Z_0|>\varepsilon/b_a\}}\}\longrightarrow0.
\]
Their total variance converges to $\sigma^2\norm{v}^2/2$. The independent triangular-array CLT applies. For the countably infinite row, truncate each row so that its remaining conditional variance is arbitrarily small, then use the finite-row theorem and Chebyshev's inequality. Cramer--Wold proves the result. Applying this argument to an arbitrary sequence $a\downarrow0$ on the same probability-one design event gives the stated limit.
\end{proof}

\begin{corollary}[Non-Gaussian errors in the AR(1) examples]
Replace the Gaussian errors in the preceding two positive AR(1) examples by any nondegenerate centered iid errors of finite variance, independent of the design. Every positive stationary moment still fails, the same lower-tail exponents $1/2$ and $\nu/2$ still hold (with different constants), and the normalized stationary law has a quenched weak Gaussian limit. In particular one may take bounded Rademacher errors.
\end{corollary}
\begin{proof}
The deterministic inequality gives, with $\ell=(1-\rho)^{-1}$,
$|aP_mX_{-m}|\ge a\ell|P_m|$. Substitute this coefficient into Lemma~\ref{up:lem:antic} instead of using a conditional Gaussian variance. The product-moment bounds in Appendices~\ref{app:ar-product}--\ref{app:ar-logweibull} and the tail events in Appendix~\ref{app:ar-tail} give the same moment and tail conclusions. The score remains a square-integrable martingale difference. The quenched limit follows from Theorem~\ref{up:thm:quenched}.
\end{proof}

\section{Exact finite-history Gaussian inference}\label{app:exact}
The following algebraic consequence provides central inference before stationarity. It uses a conditional Gaussian linear-model calculation.
Let $A_i=I-aH_i$, $T_n=A_n\cdots A_1$, and
\[
 Q_0=0,\qquad Q_n=A_nQ_{n-1}A_n^\top+a^2H_n.
\]
For a deterministic initialization and Gaussian errors independent of the design,
\[
 \beta_n-T_n\beta_0=(I-T_n)\beta^*+\varepsilon_n,
 \qquad\varepsilon_n\mid X_{1:n}\sim N_p(0,\sigma^2Q_n).
\]
If $Q_n\succ0$, then for known $\sigma^2$ the set
\[
 \left\{b:
 [\beta_n-T_n\beta_0-(I-T_n)b]^\top Q_n^{-1}
 [\beta_n-T_n\beta_0-(I-T_n)b]
 \le\sigma^2\chi^2_{p,1-\alpha}\right\}
\]
has exact conditional coverage $1-\alpha$, without a burn-in assumption. It is a bounded ellipsoid when $I-T_n$ is nonsingular. The center then is
\[
 \widetilde\beta_n=(I-T_n)^{-1}(\beta_n-T_n\beta_0).
\]
These facts follow by iterating the linear recursion; $Q_n$ is the Gram matrix of the noise coefficients. Rank-one updates require $O(p^2)$ storage and work per update for $T_n,Q_n$, with a linear solve when reporting. Overlapping residual-variance calibration gives the asymptotic coverage established in Appendix~\ref{app:extensions}; a disjoint error block yields the exact finite-sample pivot below.

\subsection{Exact calibration from a disjoint error block}
Suppose the Gaussian error assumptions of Theorem~\ref{thm:gaussian} hold. Take $m>p$ calibration observations whose indices are disjoint from the SGD history $1{:}n$, and whose design matrix has full column rank. Let $s_m^2$ be their OLS residual sum of squares divided by $m-p$. Conditional on the entire design,
$(m-p)s_m^2/\sigma^2\sim\chi^2_{m-p}$ and this variable is independent of $\varepsilon_n$ in Appendix~\ref{app:exact}. Therefore, writing $w_n(b)=\beta_n-T_n\beta_0-(I-T_n)b$, the exact pivot is
\begin{equation}\label{eq:exactF}
\frac{w_n(\beta^*)\trans Q_n^{-1}w_n(\beta^*)}{p s_m^2}
\sim F_{p,m-p}\quad\text{conditionally on the design}.
\end{equation}
Replacing the chi-square threshold in the known-variance set by $p s_m^2 F_{p,m-p;1-\alpha}$ gives exact conditional coverage, where $F_{d_1,d_2;u}$ is the $u$-quantile of the $F$ law with degrees of freedom $(d_1,d_2)$. Designs in the two blocks may be dependent. The proof uses the orthogonal projection decomposition of Gaussian calibration errors and the ratio of independent chi-square variables. Full-rank calibration and disjoint independent regression errors ensure the exact pivot.

\section{Feasible calibration for a dependent score}\label{app:hac}
The general Lyapunov limit in Theorem~\ref{thm:clt} also permits feasible central inference. Suppose $V\succ0$ and estimators satisfy $\widehat H_m\to H$ and $\widehat\Omega_m\to\Omega$ in probability. Define $\widehat V_m$ by
\[
\widehat H_m\widehat V_m+\widehat V_m\widehat H_m=\widehat\Omega_m.
\]
On the event of positive definite matrices, continuity of the Lyapunov inverse gives $\widehat V_m\to V$. For a stationary iterate or one satisfying Corollary~\ref{cor:startup}, Slutsky's theorem yields
\[
\frac{(\beta_{n(a)}-\beta^*)\trans\widehat V_m^{-1}(\beta_{n(a)}-\beta^*)}{a}
\Rightarrow\chi_p^2,
\qquad m=m(a)\to\infty.
\]
The estimator and iterate may be dependent. Feasibility rests on consistent estimation of the input score's long-run covariance, which can be established directly from the data process.

For the experiment we use the residual-score Bartlett estimator \citep{neweywest1987}. With OLS fit $\widetilde\beta_m$, put $\widehat\xi_i=X_i(Y_i-X_i\trans\widetilde\beta_m)$ and center these scores by their sample average. Set
\[
\widehat\Gamma_k=\frac1m\sum_{i=k+1}^m
(\widehat\xi_i-\overline{\widehat\xi})(\widehat\xi_{i-k}-\overline{\widehat\xi})\trans,
\quad
\widehat\Omega_{m,\ell}=\widehat\Gamma_0+
\sum_{k=1}^{\ell}\left(1-\frac{k}{\ell+1}\right)(\widehat\Gamma_k+\widehat\Gamma_k\trans).
\]
The fixed-parameter stable Gaussian VAR/AR input model in Appendix~\ref{app:experiments} has geometric dependence and all input moments. For completeness, consistency with $\ell=O(m^{1/4})\to\infty$ follows directly. For the true score, each entry of $\xi_i\xi_{i-k}\trans$ has uniformly bounded variance. Its autocovariances are bounded on lags $|h|\le k$ and decay geometrically in $|h|-k$ beyond that range: the underlying Gaussian Markov process is geometrically mixing and the score products have moments of every order. Consequently each lag-$k$ sample covariance has $L^2$ fluctuation $O(\sqrt{(k+1)/m})$. Summing over $k\le\ell$ gives $O(\ell^{3/2}/\sqrt m)=o(1)$. Absolute summability of the population score covariances removes Bartlett truncation bias by dominated convergence. Boundary bias is $O(\ell^2/m)$, also negligible.

OLS error is $O_P(m^{-1/2})$ in this model, because the sample Gram converges to $H\succ0$ and score averages have variance $O(m^{-1})$. Cauchy--Schwarz, uniform bounded input moments, and Markov's inequality bound the sum of residual-replacement errors over all $\ell$ lags by $O_P(\ell/\sqrt m+\ell/m)=o_P(1)$. Centering the scores contributes $O_P(\ell/\sqrt m)$ by the same bounds. This proves consistency for the stated Gaussian VAR/AR model and bandwidth growth.

\section{Comparison with temporal assumptions in prior work}\label{app:comparison}
Table~\ref{tab:literature} summarizes the objects and assumptions compared below.

\begin{table}[ht]\centering\small
\caption{Different objects and assumptions in closely related theories. The surrounding text examines the temporal conditions for the refresh example.}\label{tab:literature}
\setlength{\tabcolsep}{4pt}
\begin{tabular}{p{.28\linewidth}p{.32\linewidth}p{.33\linewidth}}
\toprule
Representative work & Setting or requirement & Result and relation to this paper \\
\midrule
\citet{chen2021,dieuleveut2020} & Stationary small-step analysis & Law characterization and moment expansions, respectively; our laws may have all moments infinite. \\
\citet{khusainov2025} & Iid, bounded designs; finite third score moment & Finite-history convex-distance bound; our Gaussian-error TV expansion permits dependent unbounded designs. \\
\citet{guo1997,durmus2021markov} & Temporal energy or drift/growth control & Product-moment stability; our PSD argument gives pathwise contraction. \\
\citet{roitershtein2007,damek2026} & Markov-modulated or iid affine recursions & Power-law tails; our examples combine all-moment failure with a Gaussian limit. \\
\citet{mou2020} & Fixed-step averaged linear SA & CLT and correction for averages; our correction is for the stationary last iterate. \\
\citet{merad2023,wang2026} & Invariant-law convergence or quantitative SA with moment control & TV to stationarity or Wasserstein Gaussian bounds; our expansion is from stationarity to Gaussian. \\
\bottomrule
\end{tabular}
\end{table}

\paragraph{Stationary Gaussian characterization.}
\citet{chen2021} study a deterministic drift with additive iid centered noise of finite covariance. Under their smooth strongly convex or Hurwitz linear drift assumptions, the stationary scaled iterate has a finite second moment. Our population-drift representation has noise $\xi_i-(H_i-H)e_{i-1}$, which depends on the iterate and the temporally dependent design. Their characterization therefore supplies the Gaussian/Lyapunov framework but does not cover the infinite-moment stationary laws constructed here.

\paragraph{Temporal energy.}
Theorem 2 of \citet{guo1997} considers a deterministic uniformly stable linear filter of a $\phi$-mixing source, together with a bounded deterministic term, as specified in their equation (14). Write $\zeta_i$ for that source. Their equation (15) requires positive constants $\alpha,M,K$ such that
\[
\E\exp\left\{\alpha\sum_{j=1}^n\|\zeta_{i_j}\|^2\right\}
\le M e^{Kn}
\]
for every $n\ge1$ and increasing nonnegative integer sequence $i_1<\cdots<i_n$. Within this class, their excitation condition (17) characterizes the $L^p$ exponentially stable family of update products. The direct refresh source isolates the role of joint temporal energy: retaining a single Gaussian value for $n$ observations gives
\[
\E e^{\alpha\sum_{i=1}^n X_i^2}
\ge\rho^{n-1}\E e^{\alpha nG^2}=\infty\quad\text{when }\alpha n\ge1/2.
\]
Thus the direct refresh representation lies outside this joint-energy condition despite its Gaussian marginal. This checks the applicability of that sufficient representation, rather than treating its source condition as necessary for every representation.

Likewise, the refresh design has $\E(X_i\mid\mathcal F_{i-1})=\rho X_{i-1}$, so it fails the conditional design martingale assumption in \citet{shen2026}. Its conditional second moment is $\rho X_{i-1}^2+(1-\rho)$, which is unbounded, despite standard Gaussian marginals and uniform geometric ergodicity. The frameworks are therefore not nested. \citet{wang2026} separate an iid component from a finite-state Markov component in their quantitative setting; the unbounded refresh Gram process is not covered merely by its uniform ergodicity. These distinctions locate the coexistence result within complementary theories of dependent-data SGD.

\paragraph{Unbounded-state product stability.}
\citet{durmus2021markov} allow unbounded coefficients through a super-Lyapunov drift and a matching growth bound. In their Assumption UE1, a Lyapunov function $\mathcal V$ obeys $P\mathcal V\le e^{-c(\log\mathcal V)^\delta}\mathcal V$ outside a sublevel set; Assumption A1 controls the matrix entries by a power of $\log\mathcal V$. For the refresh state $x$,
\[
P\mathcal V(x)=\rho\mathcal V(x)+(1-\rho)\E\mathcal V(G)\ge\rho\mathcal V(x).
\]
An unbounded $\mathcal V$ therefore cannot satisfy this drift: dividing by $\mathcal V(x)$ along levels tending to infinity would give $\rho\le e^{-c(\log\mathcal V(x))^\delta}\to0$. A bounded $\mathcal V$ cannot meet their growth condition for the unbounded Gram function $x^2$. Uniform geometric ergodicity thus coexists here with failure of this route to product-moment stability.

\paragraph{Tail results and the object being approximated.}
\citet{roitershtein2007}, Assumption 1.2(A4)--(A5), use a bounded additive coefficient and a multiplicative coefficient bounded above and away from zero, together with Markov-renewal conditions, to obtain power-law tails. The refresh recursion has unbounded coefficients and a multiplier arbitrarily close to zero; its all-moment failure arises from a different regime. \citet{damek2026} provide an irreducible-proximal treatment of iid symmetric-matrix SGD recursions, refining the justification in \citet{gurbuzbalaban2021}. Our scalar persistent constructions use direct anti-concentration and amplification bounds. Meanwhile, \citet{durmus2021tails}, Assumption A1 and Example 1, show that bounded iid matrices with sub-Gaussian additive noise can already obstruct exponential concentration. Our TV approximation controls absolute probabilities of all measurable events despite failure of every positive stationary moment.

\section{Experimental protocol and complete diagnostics}\label{app:experiments}
\subsection{Replications, history lengths and numerical representation}
The declared main grid contains 68 configurations. Scalar refresh designs use $\rho\in\{0,.7,.9\}$ and $a\in\{.04,.02,.01,.005,.0025\}$, with two independent seeds and 12,000 histories per seed. Refresh designs start exactly in their stationary Gaussian marginal. Their history length is
\[
n(a)=\left\lceil\frac{\max\{24,2\log^2(1/a)\}}a\right\rceil.
\]
For the refresh design, this schedule satisfies both the weak-limit condition $an(a)/\log(1/a)\to\infty$ and the sharper finite-history condition in Corollary~\ref{cor:startup-tv}. A zero-initialized paired recursion starts halfway through each history and uses the same final observations. The paired histories measure initialization sensitivity of the reported central probabilities. Rates and seeds use independent histories, except for the explicit paired histories. Pareto and baseline vector comparisons use learning time 24 as a finite approximation, with paired histories of learning time 12; refresh, log-Weibull and feasible-vector experiments use growing learning times.

The largest paired grid discrepancy among individual scalar configurations was $2.11\times10^{-5}$ for the Pareto design at $\rho=.7,a=.04$; among individual refresh seeds it was $7.35\times10^{-7}$. All per-configuration paired diagnostics are saved.

For independent Gaussian errors with variance one, propagate the exact finite-history variance and the coefficient of initialization:
\begin{equation}\label{eq:qsim}
Q_i=(1-aX_i^2)^2Q_{i-1}+a^2X_i^2,\quad Q_0=0,
\qquad T_i=(1-aX_i^2)T_{i-1},\quad T_0=1.
\end{equation}
The simulation stores $\log Q_i$ via log-add-exp, the sign and log magnitude of $T_i$, and the analogous finite sum of $a^2X_i^4$. It never clips a coefficient or deletes a replication. For $v_r=Q_r/a$, conditional integration gives
\[
\widehat F_a(x)=R^{-1}\sum_{r=1}^R\Phi(x/\sqrt{v_r}),\qquad
\widehat C_{1-\alpha}=R^{-1}\sum_r\left[2\Phi\left\{z_{1-\alpha/2}\sqrt{v_{\rm ref}/v_r}\right\}-1\right],
\]
where $v_{\rm ref}=1/2$ or $(1+aJ)/2$. Monte Carlo standard errors are sample standard deviations of these bounded conditional probabilities divided by $\sqrt R$. Saved summaries include 90\%, 95\%, and 99\% coverage, widths, expansion counts, startup coefficients, and quantiles of the conditional correction. The grid CDF uses float64 evaluation. Bootstrap resampling operates on centered float32 CDF values for memory efficiency, over the full 1,601-point grid; the reported 95\% band is the 95th percentile of the resampled suprema. It quantifies Monte Carlo error on the finite grid only. Histories are resampled as units, so the Gaussian errors are not spuriously treated as extra replications.

\begin{table}[t]\centering\small
\caption{Refresh design, $\rho=.7$, 24,000 histories. Coverage is in percent. TV estimates and the linearized Monte Carlo standard error of corrected TV are in units of $10^{-4}$. The empirical conditional Gaussian mixture defines the estimated finite-history marginal law. Full outputs also include coverage standard errors, seed-wise estimates and CDF diagnostics.}\label{tab:refresh}

\begin{tabular}{rrrrrr}
\toprule
$a$ & Base coverage & Corrected coverage & TV & Corrected TV & MCSE (corr.) \\
\midrule
0.04 & 94.3885 & 95.0766 & 123.130 & 17.859 & 0.469 \\
0.02 & 94.6849 & 95.0293 & 65.035 & 6.488 & 0.199 \\
0.01 & 94.8373 & 95.0095 & 33.976 & 2.050 & 0.084 \\
0.005 & 94.9167 & 95.0027 & 17.492 & 0.588 & 0.035 \\
0.0025 & 94.9579 & 95.0009 & 8.872 & 0.185 & 0.013 \\
\bottomrule
\end{tabular}

\end{table}

The scalar Kolmogorov first-order constant follows from Theorem~\ref{up:thm:rate}, which implies that the corrected CDF is uniformly $o(a)$ away from the true CDF. Differentiating $\Phi(x/\sqrt{(1+aJ)/2})$ gives a maximal absolute derivative $J\phi(1)/2$, hence $3\phi(1)/4$ in the refresh example. Saved finite-grid CDF diagnostics supplement the direct TV estimates plotted in Figure~\ref{fig:refresh}.

\subsection{Direct TV integration and conditional correction}
Let $v_r=Q_n^{(r)}/a$ for the $R=24000$ refresh histories, and define
\[
\widehat f_a(x)=R^{-1}\sum_{r=1}^R\frac1{\sqrt{v_r}}\phi(x/\sqrt{v_r}),\qquad
\widehat d_{\TV}=\frac12\int_{\R}|\widehat f_a(x)-g_a(x)|\,dx,
\]
where $g_a$ is the baseline or corrected Gaussian density. This is a Monte Carlo estimate of the finite-history marginal-law TV. For a reference variance $v_*$, the density ratio as a function of $z=x^2$ is
\[
\widehat f_a(\sqrt z)/g_a(\sqrt z)
=R^{-1}\sum_r\sqrt{v_*/v_r}\exp\{z(1/v_*-1/v_r)/2\}.
\]
Its logarithm is convex, so the densities have at most two positive crossings unless they coincide identically. We locate the minimum through the mixture's posterior mean precision and bracket crossings in $\log(x/\sqrt{v_*})$, using log variances throughout. Gaussian CDF differences integrate each sign interval. The omitted positive-tail integral is bounded by the sum of the mixture and reference tail probabilities; all reported bounds are below $10^{-13}$. Analytic Gaussian comparisons, independent adaptive quadrature, boundary-crossing cases and extreme log-variance checks verify the implementation. Every positive-weight history is retained.

For Monte Carlo uncertainty, integrate each component-minus-reference density over the fitted sign intervals. The sample standard deviation of these bounded per-history contributions divided by $\sqrt R$ is the linearized standard error. In the $\rho=.7$ endpoint configurations, 199 whole-history bootstrap resamples re-estimate crossings and TV; their standard deviations agree with the linearized errors. Separate-seed estimates and paired half-history checks are supplied. Finite-history and Monte Carlo approximations are documented separately from numerical integration error.

Convexity implies $\E d_{\TV}(\widehat f_a,g_a)\ge d_{\TV}(f_{a,n},g_a)$, where $f_{a,n}=\E\widehat f_a$ is the finite-history marginal density. To assess sampling sensitivity, we also integrate the distance between the two independent 12,000-history seed mixtures. At $\rho=.7$, these split distances are $4.41\times10^{-5}$ at $a=.01$ and $2.40\times10^{-6}$ at $a=.0025$, compared with pooled corrected TV estimates $2.05\times10^{-4}$ and $1.85\times10^{-5}$. Across all 15 refresh configurations, the split distance is $.009$--$.215$ times the corrected TV estimate. Adaptive quadrature and a separate partitioned integration agree within $10^{-10}$; all results are saved. These are descriptive sampling diagnostics, distinct from the paired checks of history length; a single split does not bound the plug-in bias.

The finite-history conditional correction obeys
\[
\mathcal R_{a,i}=(1-aX_i^2)^2\mathcal R_{a,i-1}+a^2X_i^4,\qquad \mathcal R_{a,0}=0.
\]
Its median and quantiles track concentration toward $J=1.5$ without averaging a stationary iterate or its square. Table~\ref{tab:conditional-correction} reports the fraction farther than $.5$ from $J$. Histories are independently generated across step sizes, so these are distributional diagnostics of the conditional expansion.
Together with the propagated $T_i$ and paired histories, this separates loss of initialization from concentration of $\mathcal R_{a,n}/a$ near $J$. Increasing history length addresses the former; the latter is a small-step property of the design law. Neither empirical diagnostic replaces the asymptotic conditions of Corollary~\ref{cor:startup-tv}.
\begin{table}[t]\centering\small
\caption{Finite-history conditional correction $\mathcal R_{a,n}/a$ in the refresh design, 24,000 histories per row. Quantiles describe between-history variation; $J=1.5$ for every persistence.}\label{tab:conditional-correction}

\begin{tabular}{rrrrrr}
\toprule
$\rho$ & $a$ & 10th percentile & Median & 90th percentile & $\widehat P(|\mathcal R_{a,n}/a-1.5|>.5)$ \\
\midrule
0 & 0.04 & 0.975 & 1.414 & 2.389 & 0.3043 \\
0 & 0.02 & 1.078 & 1.442 & 2.102 & 0.1746 \\
0 & 0.01 & 1.176 & 1.466 & 1.941 & 0.0897 \\
0 & 0.005 & 1.254 & 1.476 & 1.811 & 0.0381 \\
0 & 0.0025 & 1.318 & 1.488 & 1.719 & 0.0091 \\
0.7 & 0.04 & 0.653 & 1.095 & 2.207 & 0.5507 \\
0.7 & 0.02 & 0.786 & 1.202 & 2.110 & 0.4276 \\
0.7 & 0.01 & 0.913 & 1.285 & 2.049 & 0.2943 \\
0.7 & 0.005 & 1.030 & 1.354 & 1.967 & 0.1669 \\
0.7 & 0.0025 & 1.135 & 1.403 & 1.867 & 0.0777 \\
0.9 & 0.04 & 0.421 & 0.838 & 2.033 & 0.7159 \\
0.9 & 0.02 & 0.538 & 0.959 & 2.004 & 0.6342 \\
0.9 & 0.01 & 0.674 & 1.074 & 2.010 & 0.5298 \\
0.9 & 0.005 & 0.811 & 1.189 & 2.009 & 0.3962 \\
0.9 & 0.0025 & 0.937 & 1.292 & 1.969 & 0.2536 \\
\bottomrule
\end{tabular}

\end{table}

\subsection{Pareto and log-Weibull AR designs}
For the raw design $\widetilde X_i=\rho\widetilde X_{i-1}+\eta_i$, normalize by its population second moment. The Pareto innovation satisfies $\Prob(\eta>u)=u^{-3}$, $u\ge1$, giving
\[
h_{\rm raw}=\frac{3/4}{1-\rho^2}+\left(\frac{3/2}{1-\rho}\right)^2,
\qquad X_i=\widetilde X_i/\sqrt{h_{\rm raw}}.
\]
We use $\rho\in\{0,.7,.9\}$ and the same five steps, with 12,000 independent replications per configuration. For log-Weibull inputs, take $\rho=.7$, $\eta=\exp(cE^{1/\nu})$, $\nu=1.5$, $c\in\{.5,1\}$ and $E\sim\mathrm{Exp}(1)$, at steps $.04,.01,.0025$, also with 12,000 replications. The tail is $\exp[-(\log u/c)^\nu]$, so the no-moment argument holds at both scales. Population innovation moments are computed by scaled adaptive quadrature of
\[
\E\eta^k=\int_0^\infty\exp(kct^{1/\nu}-t)\,dt.
\]
Raw AR moments $M_k$ obey $(1-\rho^k)M_k=\sum_{j=0}^{k-1}\binom{k}{j}\rho^jM_j\E\eta^{k-j}$, $M_0=1$. Thus the normalized correction is $J=M_4/(2M_2^2)$. Quadrature tolerances and errors are saved. AR designs start at their population mean and warm up for 512 steps; their omitted initial contribution is attenuated by $\rho^{512}$. No normalization is fitted to simulated paths.

\begin{table}[t]\centering\small
\caption{Selected Pareto and log-Weibull configurations, 12,000 histories each. The parameter is $\rho$ for Pareto and $c$ for log-Weibull. Coverage is in percent. An expansion path contains at least one factor with $|1-aX_i^2|>1$ anywhere in its history. Pareto-$3$ does not meet the correction theorem's fourth design moment, so no corrected result is assigned. Every configuration is retained in the supplied results.}\label{tab:other}

\begin{tabular}{llrrrr}
\toprule
Model & Parameter & $a$ & Coverage & Corrected coverage & Expansion paths \\
\midrule
pareto & 0 & 0.04 & 93.1154 & -- & 3432 \\
pareto & 0 & 0.01 & 94.0687 & -- & 1829 \\
pareto & 0 & 0.0025 & 94.5519 & -- & 941 \\
pareto & 0.7 & 0.04 & 94.6931 & -- & 226 \\
pareto & 0.7 & 0.01 & 94.9047 & -- & 87 \\
pareto & 0.7 & 0.0025 & 94.9703 & -- & 36 \\
pareto & 0.9 & 0.04 & 94.7574 & -- & 10 \\
pareto & 0.9 & 0.01 & 94.9384 & -- & 3 \\
pareto & 0.9 & 0.0025 & 94.9841 & -- & 2 \\
logweibull & 0.5 & 0.04 & 94.7422 & 95.0008 & 0 \\
logweibull & 0.5 & 0.01 & 94.9358 & 95.0003 & 0 \\
logweibull & 0.5 & 0.0025 & 94.9839 & 95.0000 & 0 \\
logweibull & 1 & 0.04 & 94.5327 & 97.8303 & 967 \\
logweibull & 1 & 0.01 & 94.8405 & 95.9553 & 631 \\
logweibull & 1 & 0.0025 & 94.9314 & 95.2344 & 350 \\
\bottomrule
\end{tabular}

\end{table}

\subsection{Noncommuting design and serial score covariance}
The vector experiment uses
\begin{equation}\label{eq:vectorconfig}
\widetilde X_i=B\widetilde X_{i-1}+\eta_i,\quad
B=\begin{pmatrix}.6&.15\\0&.4\end{pmatrix},\quad
\Cov(\eta_i)=\begin{pmatrix}1&.3\\.3&1\end{pmatrix},
\quad Z_i=.7Z_{i-1}+\sqrt{1-.7^2}\,\zeta_i,
\end{equation}
where Gaussian innovations are independent between design and errors, and $\zeta_i$ is standard normal. Solve $H_{\rm raw}=BH_{\rm raw}B\trans+\Cov(\eta_i)$ and rescale $X_i=s\widetilde X_i$, $s=\sqrt{2/\tr H_{\rm raw}}$. Initialize both input states exactly in their stationary Gaussian laws. With $A=(I-.7B)^{-1}$, independence gives
\begin{equation}\label{eq:vectoromega}
\Omega=AH+HA\trans-H,\qquad HV+VH=\Omega.
\end{equation}
The population matrices are
\begin{align*}
H&\simeq\begin{pmatrix}1.18802063&.33333890\\.33333890&.81197937\end{pmatrix},&
\Omega&\simeq\begin{pmatrix}3.07622956&.90851551\\.90851551&1.44351887\end{pmatrix},\\
V&\simeq\begin{pmatrix}1.26604476&.10208074\\.10208074&.84698206\end{pmatrix}.&&
\end{align*}
The baseline vector runs use all five steps and 12,000 independent replications, with zero SGD initialization and $\lceil24/(a\lambda_{\min}(H))\rceil$ updates. The feasible-HAC runs use steps $.04,.01,.0025$ and 6,000 replications. Each processes 8,192 calibration observations, followed by $\lceil\max\{24,2\log^2(1/a)\}/(a\lambda_{\min}(H))\rceil$ observations before recording the endpoint. SGD processes both blocks; the calibration score estimates use only the earlier block. The last 2,048 or all 8,192 calibration observations are used with $\ell=\lfloor c_m m^{1/4}\rfloor$, $c_m\in\{.5,1,2\}$, yielding $(3,6,13)$ or $(4,9,19)$. Matrices are symmetrized only through their defining algebra; no eigenvalue clipping or oracle replacement is used. A non-positive estimate is counted as a failed, empty confidence set.

\begin{table}[t]\centering\small
\caption{Dependent-score ellipse coverage in percent, with Monte Carlo standard errors in parentheses, from 6,000 replications. HAC column headings give calibration size and bandwidth. All runs and all bandwidth choices are in the saved summaries.}\label{tab:hac}

\begin{tabular}{rrrrrr}
\toprule
$a$ & Oracle & Lag zero & HAC $(2048,13)$ & HAC $(8192,9)$ & HAC $(8192,19)$ \\
\midrule
0.04 & 94.92 (0.28) & 77.60 (0.54) & 93.60 (0.32) & 93.55 (0.32) & 94.12 (0.30) \\
0.01 & 95.35 (0.27) & 76.62 (0.55) & 93.92 (0.31) & 93.82 (0.31) & 94.50 (0.29) \\
0.0025 & 95.23 (0.28) & 76.17 (0.55) & 93.60 (0.32) & 93.78 (0.31) & 94.38 (0.30) \\
\bottomrule
\end{tabular}

\end{table}

For the normalized endpoint error $u=(\beta_n-\beta^*)/\sqrt a$, the oracle statistic is $u\trans V^{-1}u$ and the lag-zero statistic is $2\|u\|^2$, compared with $\chi^2_{2,.95}$. For a Gaussian limit the latter has coverage
\[
\frac1{2\pi}\int_0^{2\pi}\left[1-\exp\{-\chi^2_{2,.95}/(4v(\theta))\}\right]d\theta,
\quad v(\theta)=(\cos\theta,\sin\theta)V(\cos\theta,\sin\theta)\trans.
\]
A deterministic 100,000-angle quadrature gives $.760097279998$, quantifying the central miscalibration caused by ignoring serial score covariance.

To isolate finite-bandwidth tapering, define the population Bartlett benchmark
\begin{equation}\label{eq:population-bartlett}
\Gamma_k=(.7B)^kH,\quad
\Omega_\ell=H+\sum_{k=1}^\ell\left(1-\frac{k}{\ell+1}\right)(\Gamma_k+\Gamma_k\trans),
\qquad HV_\ell+V_\ell H=\Omega_\ell.
\end{equation}
Let $\lambda_1,\lambda_2$ be the generalized eigenvalues of $(V,V_\ell)$. For $U\sim N_2(0,V)$, the coverage of $U\trans V_\ell^{-1}U\le\chi^2_{2,.95}$ is
\begin{equation}\label{eq:bartlett-coverage}
C_\ell=\frac1{2\pi}\int_0^{2\pi}\left[1-\exp\left\{-\frac{\chi^2_{2,.95}}{2(\lambda_1\cos^2\theta+\lambda_2\sin^2\theta)}\right\}\right]d\theta.
\end{equation}
Adaptive quadrature gives $C_9=.936597378164$ and $C_{19}=.943743069940$. Figure~\ref{fig:hac}(b) plots this population Gaussian reference alongside the feasible estimates. It holds the bandwidth fixed and uses the known model to quantify tapering; finite-sample covariance estimation and finite-step approximation remain separate sources of error. The exact matrices, eigenvalues and quadrature errors for $\ell=3,\ldots,19$ are supplied with the outputs.

\subsection{Non-Gaussian errors and bounded process diagnostics}
At $\rho=.7$ and $a\in\{.04,.01,.0025\}$, actual scalar SGD is run with Gaussian, Rademacher, and variance-one Student-$t_5$ errors, with 12,000 independent replications each. Signed-log addition represents the state without clipping, including cancellation between opposite signs. These distributions test the finite-variance conditional CLT through marginal coverage. Actual-error coverage estimates have Bernoulli Monte Carlo standard errors.

For an OU diagnostic, save two endpoints separated by $k=\lceil .5/a\rceil$ updates, with learning-time lag $\tau=ak$. Conditional on a Gaussian-error design history, write $U_{\rm new}=T_kU_{\rm old}+\varepsilon$ with conditionally independent new noise. For $(s,t)\in\{(1,1),(1,-1),(2,.5)\}$, evaluate
\[
\E[e^{i(sU_{\rm old}+tU_{\rm new})}\mid X]
=\exp\{-\tfrac12[(s+tT_k)^2v_{\rm old}+t^2v_{\rm noise}]\}.
\]
The sum of nonnegative variance terms is evaluated in log space, so rare large variances contribute their correct bounded limiting value zero. The OU target is $\exp[-(s^2+t^2+2st e^{-\tau})/4]$. This bounded diagnostic targets two-time distributions of the OU limit without requiring stationary covariance moments.

\begin{figure}[t]\centering
\includegraphics[width=\linewidth]{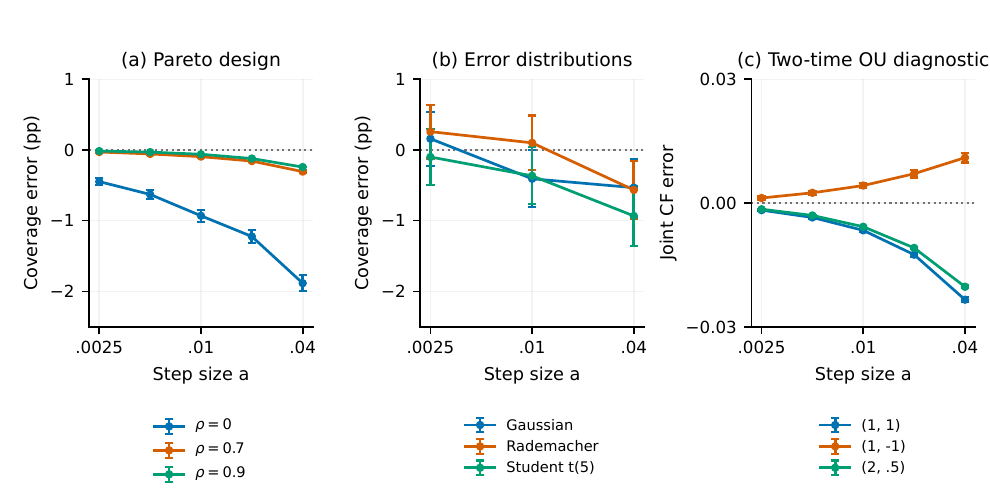}
\caption{Qualitative stationary and process predictions. (a,b) Coverage error relative to 95\%, in percentage points, on the same fixed axis range: Pareto AR designs and actual Gaussian/Rademacher/Student-$t_5$ errors, respectively. (c) Bounded two-time characteristic-function errors relative to the OU prediction, with a zero reference. Bars show 1.96 Monte Carlo standard errors. Log-Weibull results appear in Tables~\ref{tab:other} and~\ref{tab:stress}.}\label{fig:additional}
\end{figure}

\subsection{Observed inference before burn-in}
Set $\beta^*=0$, $\beta_0=1$ and use the refresh design with $\rho=.7$. For each of $a=.04,.01$, simulate 48,000 independent design/error trajectories and record endpoints near learning times $an=.05,.25,1,4,24$. Before SGD starts, generate 65 calibration observations on the same design stream, fit scalar OLS and compute its residual variance $s^2$ with 64 degrees of freedom. Regression errors are actual independent standard Gaussian draws in both blocks. The calibration block and SGD history have disjoint errors and potentially dependent designs.

All methods use the same endpoint and design history. Alongside SGD, propagate $T_n,Q_n$ from \eqref{eq:qsim}, and put $w_n=\beta_n-T_n\beta_0$. The stationary interval uses center $\beta_n$ and radius $z_{.975}\sqrt{a/2}$. The history-corrected center is $w_n/(1-T_n)$, with radius $z_{.975}\sqrt{Q_n}/|1-T_n|$ for known variance, or $t_{64,.975}s\sqrt{Q_n}/|1-T_n|$ for estimated variance, where $t_{d,u}$ is the $u$-quantile of Student's $t$ law with $d$ degrees of freedom. We record actual coverage indicators and interval widths. Figure~\ref{fig:startupshock} shows $a=.01$, and Table~\ref{tab:startup} reports both step sizes. Endpoint estimates share trajectories across times; error bars use independent trajectories as the replication unit.

\begin{figure}[t]\centering
\includegraphics[width=\linewidth]{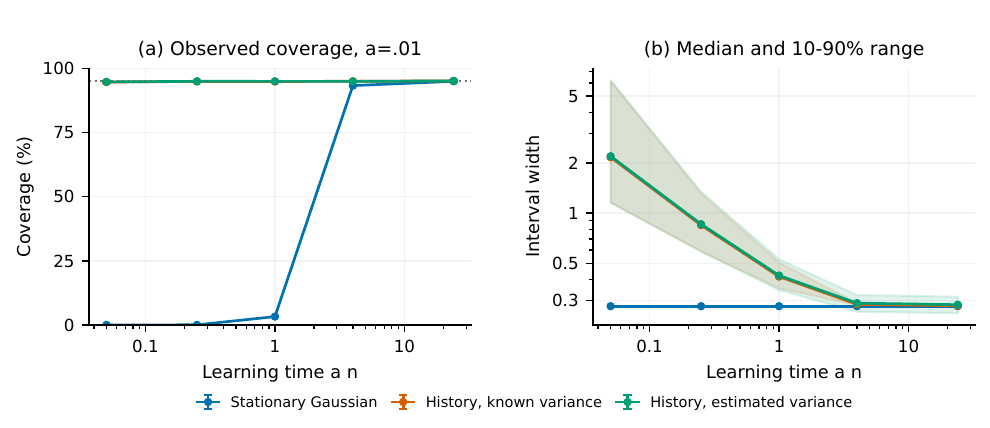}
\caption{Observed calibration and precision before burn-in, $a=.01$, 48,000 trajectories. (a) Actual 95\% coverage with 1.96 Monte Carlo standard-error bars; the dotted line is nominal. Both history-corrected curves are measured from simulated outcomes. (b) Median interval widths and 10th--90th percentile bands across the same trajectories, on a logarithmic scale. Estimated-variance intervals use an earlier block of 65 observations.}\label{fig:startupshock}
\end{figure}
\begin{table}[t]\centering\footnotesize
\caption{Actual finite-history inference, 48,000 trajectories per step size. Coverage is in percent with Monte Carlo standard errors in parentheses. Width columns describe the estimated-variance history-corrected interval.}\label{tab:startup}

\begin{tabular}{rrrrrrr}
\toprule
$a$ & $an$ & Naive & Known $\sigma^2$ & Estimated $\sigma^2$ & Median width & 90th width \\
\midrule
0.04 & 0.08 & 0.12 (0.02) & 95.03 (0.10) & 95.06 (0.10) & 3.811 & 16.863 \\
0.04 & 0.28 & 2.86 (0.08) & 94.86 (0.10) & 94.97 (0.10) & 1.781 & 4.308 \\
0.04 & 1 & 29.89 (0.21) & 95.04 (0.10) & 95.02 (0.10) & 0.884 & 1.344 \\
0.04 & 4 & 93.05 (0.12) & 95.08 (0.10) & 95.01 (0.10) & 0.594 & 0.669 \\
0.04 & 24 & 94.40 (0.10) & 95.01 (0.10) & 95.08 (0.10) & 0.577 & 0.645 \\
0.01 & 0.05 & 0.00 (0.00) & 94.69 (0.10) & 94.72 (0.10) & 2.186 & 6.242 \\
0.01 & 0.25 & 0.02 (0.01) & 94.90 (0.10) & 94.97 (0.10) & 0.857 & 1.342 \\
0.01 & 1 & 3.23 (0.08) & 94.90 (0.10) & 94.94 (0.10) & 0.423 & 0.529 \\
0.01 & 4 & 93.30 (0.11) & 94.98 (0.10) & 94.93 (0.10) & 0.289 & 0.323 \\
0.01 & 24 & 94.96 (0.10) & 95.13 (0.10) & 95.13 (0.10) & 0.283 & 0.316 \\
\bottomrule
\end{tabular}

\end{table}

\subsection{An imposed shock and its amplification}
For the normalized Pareto AR design with $\rho=.7$, $a=.01$, start at the population mean, impose one innovation $u\in[10^2,10^7]$, and use the innovation mean for the next 599 updates. Record the maximum partial log product, including the empty product. With $s=h_{\rm raw}^{-1/2}$ and
\[
m=1+\left\lfloor\frac{\log(\sqrt{a/2}\,su)}{\log(1/\rho)}\right\rfloor,
\]
the analytic lower bound is $m\log(a/2)+m(m-1)\log\rho+2m\log(su)$. The positive background increases factor magnitudes over the selected expansion interval. This deterministic construction isolates the amplification mechanism used in the tail proof.
\begin{figure}[t]\centering
\includegraphics[width=.72\linewidth]{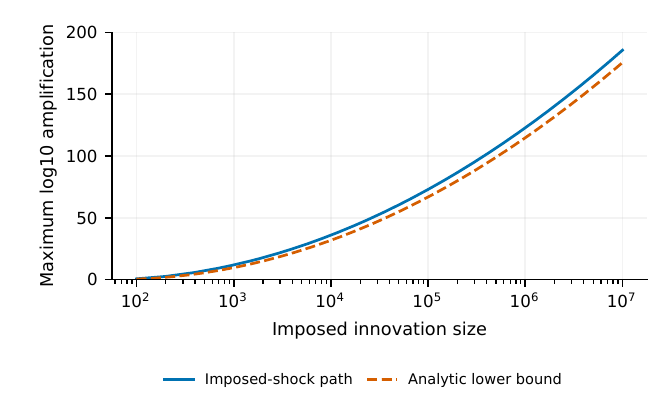}
\caption{Persistence amplifies one extreme innovation. The deterministic path and analytic lower bound describe amplification conditional on an imposed shock; the vertical axis is its base-10 logarithm.}\label{fig:shock}
\end{figure}

\subsection{Additional log-Weibull stress checks}\label{app:stress}
After the declared main grid revealed overcorrection at $c=1$, we fixed two additional configurations: 48,000 independent histories at $a=.0025$ and 24,000 at $a=.000625$, with new seeds 940001 and 940002. They retain the same normalization, $\rho=.7$, $\nu=1.5$, and growing histories. These checks assess sampling precision and the effect of a smaller learning rate. The results are summarized in the generated stress table.

\begin{table}[ht]\centering\small
\caption{Log-Weibull stress checks at $c=1$. Coverage is reported as a percentage, with Monte Carlo standard errors in percentage points.}\label{tab:stress}
\begin{tabular}{rrrrr}
\toprule
$a$ & Histories & Baseline coverage & Corrected coverage & MCSE (baseline, corrected) \\
\midrule
0.0025 & 48000 & 94.9426 & 95.2455 & (0.0026, 0.0026) \\
0.000625 & 24000 & 94.9772 & 95.0545 & (0.0031, 0.0031) \\
\bottomrule
\end{tabular}
\end{table}

The smaller step reduces both coverage errors. At $a=.000625$, baseline and corrected errors are approximately $-0.0228$ and $+0.0545$ percentage points, respectively. The population correction is $J\simeq10.8553$, whereas the sampled median of $\mathcal R_{a,n}/a$ is about $1.44$ and its 99th percentile is about $25.53$. The sampled conditional correction remains dispersed relative to its limiting coefficient, consistent with delayed entry into the first-order regime. Over the tested steps, the baseline interval remains closer to nominal coverage in the $c=1$ setting.

\paragraph{Reproducibility artifacts.}
Configuration files and the ancillary README specify the seeds, software versions, and reproduction steps. Saved analysis summaries, execution metadata and algebraic checks accompany the scripts. The README identifies the omitted replication arrays and explains how to regenerate them.

\end{document}